\documentclass{article}

\usepackage{microtype}
\usepackage{graphicx}
\usepackage{subfigure}
\usepackage{booktabs} % for professional tables

\usepackage{hyperref}

\usepackage[accepted]{icml2024}

\usepackage{amsmath}
\usepackage{amssymb}
\usepackage{mathtools}
\usepackage{amsthm}

\usepackage[capitalize,noabbrev]{cleveref}

\theoremstyle{plain}
\newtheorem{theorem}{Theorem}[section]
\newtheorem{proposition}[theorem]{Proposition}
\newtheorem{lemma}[theorem]{Lemma}

\theoremstyle{definition}
\newtheorem{definition}[theorem]{Definition}

\theoremstyle{remark}

\def \bzero{\mathbf{0}}
\def \trans{^\mathrm{T}}
\def \relu{\operatorname{Relu}}
\def \srs{\operatorname{srs}}
\def \cls{\operatorname{cls}}

\usepackage{amsmath,amsfonts,bm}

\def\eqref#1{equation~\ref{#1}}
\def\1{\bm{1}}

\def\rvb{{\mathbf{b}}}

\def\rvf{{\mathbf{f}}}

\def\rvh{{\mathbf{h}}}
\def\rvu{{\mathbf{i}}}

\def\rvm{{\mathbf{m}}}

\def\rvp{{\mathbf{p}}}

\def\rvu{{\mathbf{u}}}
\def\rvv{{\mathbf{v}}}
\def\rvw{{\mathbf{w}}}
\def\rvx{{\mathbf{x}}}
\def\rvy{{\mathbf{y}}}

\def\rmH{{\mathbf{H}}}
\def\rmI{{\mathbf{I}}}

\def\rmM{{\mathbf{M}}}

\def\rmW{{\mathbf{W}}}

\DeclareMathAlphabet{\mathsfit}{\encodingdefault}{\sfdefault}{m}{sl}
\SetMathAlphabet{\mathsfit}{bold}{\encodingdefault}{\sfdefault}{bx}{n}

\def\sH{{\mathbb{H}}}

\def\sP{{\mathbb{P}}}

\newcommand{\rnd}{\text{rnd}}
\newcommand{\fx}{\text{fix}}
\usepackage{enumerate}
\usepackage[textsize=tiny]{todonotes}

\icmltitlerunning{Breaking through the Learning Plateaus of In-context Learning in Transformer}

\begin{document}

\twocolumn[
\icmltitle{Breaking through the Learning Plateaus of In-context Learning in Transformer}



\begin{icmlauthorlist}
\icmlauthor{Jingwen Fu}{xjtu}
\icmlauthor{Tao Yang}{xjtu}
\icmlauthor{Yuwang Wang}{tsinghua}
\icmlauthor{Yan Lu}{msra}
\icmlauthor{Nanning Zheng}{xjtu}
\icmlcorrespondingauthor{Nanning Zheng}{nnzheng@mail.xjtu.edu.cn}

\end{icmlauthorlist}

\icmlaffiliation{xjtu}{National Key Laboratory of Human-Machine Hybrid Augmented Intelligence, National Engineering Research Center for Visual Information and Applications, and Institute of Artificial Intelligence and Robotics, Xi'an Jiaotong University}
\icmlaffiliation{tsinghua}{Tsinghua University}
\icmlaffiliation{msra}{Microsoft Research Asia}
\icmlkeywords{Machine Learning, ICML}

\vskip 0.3in
]



\printAffiliationsAndNotice{} % otherwise use the standard text.
\begin{abstract}

In-context learning, i.e., learning from context examples, is an impressive ability of Transformer. Training Transformers to possess this in-context learning skill is computationally intensive due to the occurrence of \textit{learning plateaus}, which are periods within the training process where there is minimal or no enhancement in the model's in-context learning capability. To study the mechanism behind the learning plateaus, we conceptually separate a component within the model's internal representation that is exclusively affected by the model's weights. We call this the “weights component”, and the remainder is identified as the “context component”. By conducting meticulous and controlled experiments on synthetic tasks, we note that the persistence of learning plateaus correlates with compromised functionality of the weights component. Recognizing the impaired performance of the weights component as a fundamental behavior that drives learning plateaus, we have developed three strategies to expedite the learning of Transformers. The effectiveness of these strategies is further confirmed in natural language processing tasks. In conclusion, our research demonstrates the feasibility of cultivating a powerful in-context learning ability within AI systems in an eco-friendly manner.
% However, the exact mechanism behind this learning process remains unclear. In this study, we aim to explore this aspect from a relatively less explored perspective, i.e., representation learning. For in-context learning, the representation becomes more complex as it can be influenced by both model weights and context examples. To study how the model weights and in-context examples affect the prediction, we conceptually isolate the component, that can only be influenced by the model's weights, from the model's inner representation. We name this component as in-weights component and the rest as in-context component.  We create a novel synthetic experimental set up, which allows to control the complexity of learning in-context component, making it possible to study how the two components interplay with each other and impact the in-context performance. We find that the in-weights component plays a significant role in the learning of the in-context component. However, in traditional training way, the the in-weights component may be overlooked, resulting in a poor performance. We propose to a training settup to synergistically learn the in-weight and in-context components and  the in-context learning performance can be significant improved. A further theoretical analysis is provided to justify the importance of our findings. Overall, our discoveries from the perspective of representation learning provide valuable insights into new approaches for enhancing in-context capacity.

\end{abstract}
% 简化摘要
% 讨论probe的关系和Transformer结构的作用
% 

\section{Introduction}

\begin{figure*}
    \centering
    \includegraphics[width=0.95 \textwidth]{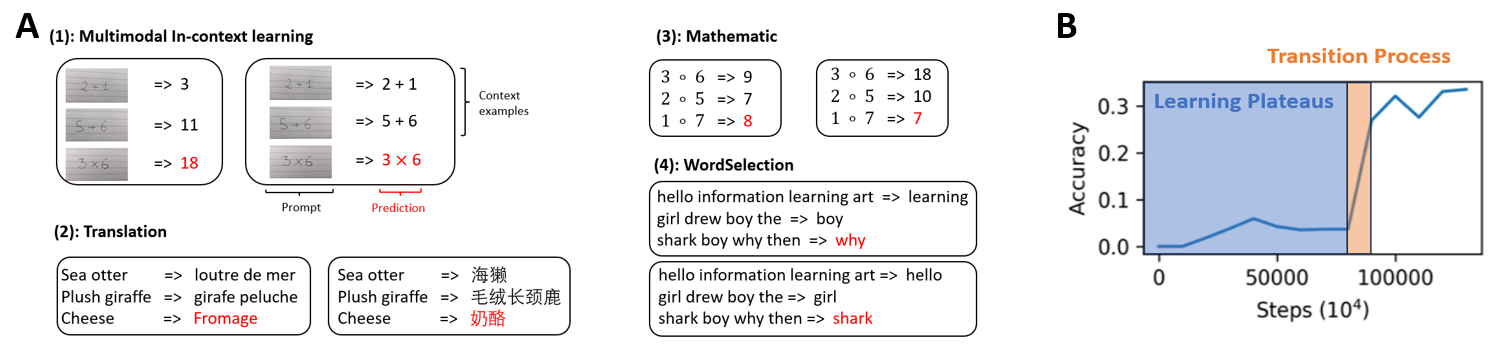}
    \caption{ \textbf{A: Examples of the in-context learning tasks.} Examples of (1) comes from \citet{alayrac2022flamingo}, Examples of (2),(3),(4) come from \citet{brown2020language}. \textbf{B: Illustration of learning plateaus and transition pattern.} We evaluate the in-context learning ability of Pythia 13B model~\citep{biderman2023pythia} trained on pile dataset~\citep{gao2020pile} using WordSelection task (Detail in Appendix~\ref{sec:nlp_app}) during the training process.  } 
    \label{fig:examples}
\end{figure*}

This paper is centered on the in-context learning ability of Transformer, which stands as one of the most significant abilities for current applications of Transformer models~\citep{brown2020language,dong2022survey,shin2022effect,min2021metaicl}. Fig.~\ref{fig:examples}A gives examples of the in-context learning tasks. The in-context learning ability has a confusing property that it is emergent when increasing training FLOPs~\citep{wei2022emergent}. By exploring the learning dynamic of in-context learning, previous works ~\citep{edelman2022inductive,michaud2023quantization,kirsch2022general,singh2023transient,reddy2023mechanistic} discover that the Transformer learns the in-context learning ability abruptly. We summarize the phenomenon during the learning process as plateaus and transition pattern, illustrated in Fig.~\ref{fig:examples}B. This pattern indicates that there is negligible or no enhancement in the in-context learning ability during the initial learning period, which we call the learning plateaus, and these plateaus are then followed by rapid and substantial gains, which are referred to as the transition process. Often, the strategy to shorten learning plateaus involves expanding the model's scale, which consequently demands greater computational resources and energy consumption. \textit{In this paper, our objective is to investigate the possibility of overcoming the learning plateaus without scaling the model's size.} Addressing this challenge could lead to a novel approach to creating environmentally sustainable intelligent systems. Since learning plateaus are not typically observed in conventional supervised learning, concentrating on the difference may be crucial for unraveling the mechanisms behind the learning plateaus.

\textbf{Weights and context components}  \quad
A principal distinction between in-context learning and traditional supervised learning lies in the fact that in-context learning outcomes are shaped by both the model's parameters and the specific context examples provided. 
In traditional supervised learning, for a given input sample $\rvx_p$ with its corresponding label $y_p$, the goal is to find a parameterized function $f_{\rvw}$ with weights $\rvw$, such that the prediction $f_{\rvw}(\rvx_p)$ is equal to the label $y_p$. In this scenario, the weights $\rvw$ hold all the information needed to perform the task at hand. Conversely, within the in-context learning paradigm, there is an additional source of information, which is the context examples $s_c$. Hence, the prediction model is represented as $f_{\rvw,s_c}(\rvx_p)$, indicating that both the weights and the context examples have the potential to affect the prediction outcome. To examine how the weights and in-context examples impact the prediction, we assume that the function $f_{\rvw,s_c}(x)$ can be \textbf{conceptual decomposite} into $f_{\rvw,s_c}(x_p)=g_{comb}(g_{weights}(\rvx_p),g_{context}(s_c))$. This decomposition allows us to separate the component $g_{weights}(\cdot)$, which is part of the Transformer solely dependent on the weights, from the component $g_{context}(\cdot)$, which is influenced by both the weights and the context examples. We refer to $g_{weights}(\cdot)$ as the weights component and $g_{context}(\cdot)$ as the context component. Owing to the design of the Transformer's architecture, there is an interaction between $g_{weights}(\cdot)$ and $g_{context}(\cdot)$. \textbf{Importantly}, conceptual decomposition implies that the decomposition isn't physical; rather, it's solely for analytical purposes.

% For regular supervised learning, given the input sample $\rvx_p$ with label $y_p$, we want to find a parameterized function $f_{\rvw}$ with weights $\rvw$, such that $y_p=f_{\rvw}(\rvx_p)$. In this situation, all the information to accomplish the tasks is stored in the weights $\rvw$. However, for the in-context learning framework, there is extra information source, the context examples $s_c$. Then, the prediction can be modeled as $f_{\rvw,s_c}(\rvx_p)$, that means both the weights and the context examples can influence the prediction.  \textbf{To study the influence of  the weights and in-context examples on the prediction, we assume that the function $f_{\rvw,s_c}(x)$ can be decomposited as $f_{\rvw,s_c}(x_p)=g(g_{weights}(\rvx_p),g_{context}(s_c))$.} In this way, we can isolate a component $g_{weights}(\rvx_p)$ that only depends on the weights from the other component $g_{context}(s_c)$ that both depend on weights and context examples. We denote $g_{weights}(\rvx_p)$ as weights component and  $g_{context}(s_c)$ as context component. Due to the weights sharing within the transformer, $g_{weights}$ and $g_{context}$ interplay with each other.

\textbf{Key observations} \quad Owing to the practical challenges of directly investigating the components, as well as the inability to regulate the complexity of these tasks, we depend on a synthetic task to conduct controlled experiments. Through conducting experiments involving tasks of varying difficulty and monitoring the performance of both the weights and the context components, we have made a critical observation: \textit{the duration of learning plateaus are often associated with a dysfunction of the weights component.}

\textbf{Break through the learning plateaus.} \quad 
Drawing from our observations, we consider the dysfunction of the weights component during learning to be the primary cause of the extended learning plateaus.   Based on this, we suggest three methods for enhancing the weights component and all these methods can mitigate the learning plateaus. The effectiveness of these methods is also verified in NLP tasks.

\textbf{Contributions} \quad
Our main contributions can be summarized as follows: \textbf{a)} We give formulations of weights and context components with a new synthetic task that enables the study of the mechanism behind learning plateaus. \textbf{b)} We study the learning process of synthetic tasks with different complexity. The experiments reveal the relation between the weights component and the learning plateaus. \textbf{c)} To further verify the causal relation between the learning plateaus and the weights component,  we propose different methods to improve the weights component and we observe the mitigating of the learning plateaus. The discoveries are verified in the NLP task.
% \paragraph{Take away} Our paper has the following key take aways for persons interested in improving the in-context learning ability of Transformer: 
% \begin{itemize}
%     \item Overtraining is essential due to the learning plateu of in-context learning.
%     \item  When the in-context learning task is hard to learn, improving the weights component is important.
% \end{itemize}
% Text dataset, label change influence, GPT v.s. Bert
\begin{figure*}
    \centering
    \includegraphics[width=0.65 \textwidth]{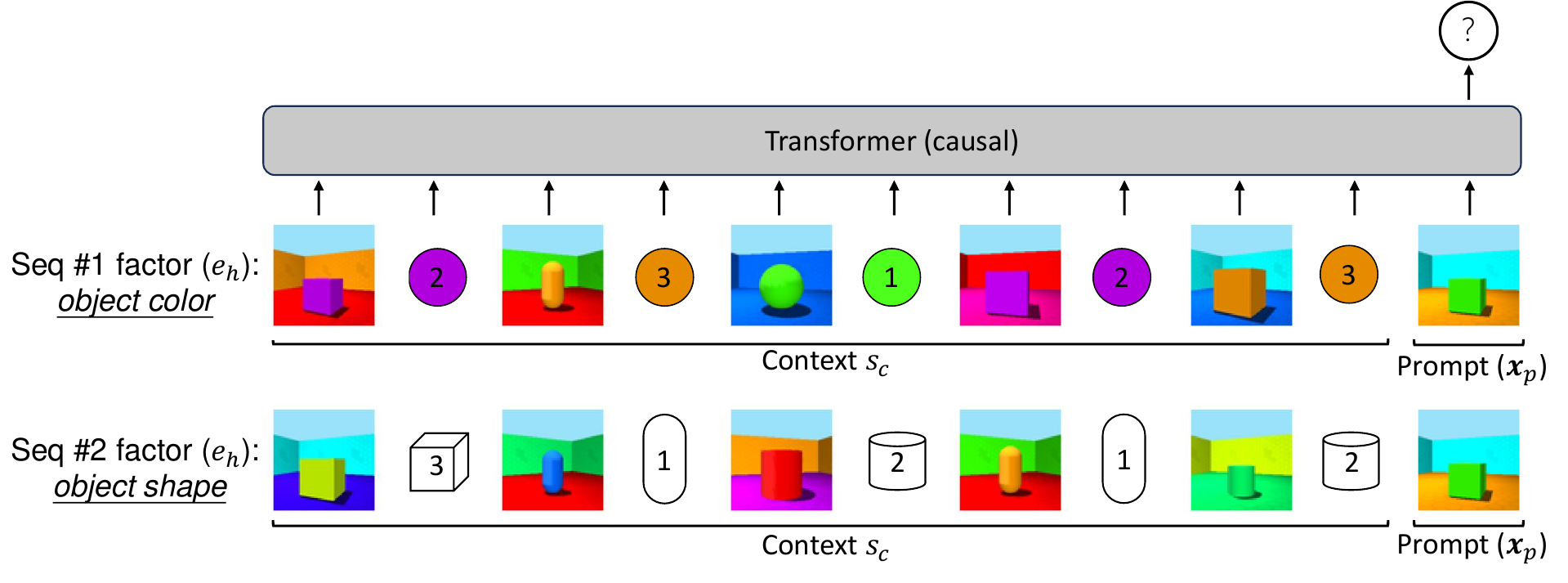}
    \caption{\textbf{Synthetic task}. In the task, Transformer is required to predict the label of $\rvx_p$ given context examples $s_c$. The images from the 3D Shapes dataset are synthesized based on six factors. The output factor is determined by the context. In this case, we provide two sequences of factors: "object color" and "object shape," respectively.}    
    \label{fig:data_vis}
\end{figure*}
\section{Related works}

In this section, we explore the works most closely related to our study. Additional related literature can be found in Appendix \ref{sec:other_related_works}. This appendix encompasses a) an examination of the weights and context components in relation to the previously proposed division of in-weights and in-context learning, b) a review of evidence from prior studies that supports the importance of both context and weights components in practical applications, and c) a compilation of related works aimed in understanding Transformers.

\textbf{Analyzing the transition phenomenon of in-context learning} \quad The emergence of in-context learning capabilities during learning has captivated numerous researchers~\citep{olsson2022context,michaud2023quantization,kirsch2022general,singh2023transient,reddy2023mechanistic}. \citet{olsson2022context} has found that this transition process is linked to the creation of inductive heads within Transformer models. Meanwhile, other researchers have determined that the occurrence of such a transition is influenced by the properties of the dataset~\citep{chan2022data}. \citet{li2022emergent,lu2023emergent} go further to associate the emergent capabilities seen during the learning with those that appear during the scaling of the model. Our research, however, is not centered on this transition phenomenon. We aim to delve into the learning plateaus phenomenon, which serves as a vital complement to understanding the intricacies of the transition phenomenon.

\textbf{Understanding of the mechanism of in-context learning} \quad Considerable research has been devoted to this vital topic, with most prior studies focusing on understanding the mechanism of in-context learning from the perspective of algorithm implementation. For example, a number of recent papers \citep{von2022transformers,dai2023can,akyurek2022learning} have described in-context learning as akin to performing gradient descent. Additional research \citep{li2023transformers,bai2023transformers} has interpreted in-context learning in terms of algorithm implementation and choice. Our study, however, takes a novel approach by examining the in-context learning mechanism through the lens of the weights and context components, offering a distinct perspective from earlier works.

\section{Experimental Design}
\label{subsec:dataset_construction}

This paper employs a synthetic task to investigate the fundamental mechanisms behind learning plateaus. The use of a synthetic task is due to the fact that the intricacies of real tasks pose challenges in monitoring and comprehending the precise factors that influence the emergence of learning plateaus.

\subsection{Dataset Construction}

We propose a task using the Shapes3D~\citep{kim2018disentangling} dataset for a more controllable study. The experimental setting is shown in Fig.~\ref{fig:data_vis}. Specifically, given a sequence of image and label pairs as context, the task involves predicting the label of the prompt image. Each image contains six different factors: object color, object shape, object scale, background color, floor color, and pose. We denote the factor as $e$ and the factor value of factor $e$ as $v^{(e)}$. For each sequence, we randomly choose a factor to generate the labels of the images, referring to this factor as the \textbf{hidden factor} $e_h$ for this sequence. For the two context sequences in Fig.~\ref{fig:data_vis}, the hidden factor of Seq \#1 is object color, and the correct label for the prompt image is 1 (object color is green). In Seq \#2, for the same prompt image, the correct label is 3 (the object shape is a cube).
% \begin{wrapfigure}[16]{r}{0.32\textwidth} 
% % 参数：{环绕位置}{图像宽度}  

%   \centering  \includegraphics[width=0.30\textwidth]{images/rep_define.png}   
%   \caption{Representation of prompt}  
% \label{fig:diff_model}
% \end{wrapfigure}
We give a formal formulation of the data generation process below.

\textbf{Notations} \quad We denote $\rvx_p$ as the prompt example with ground truth label $y_p$. The context examples are $s_c=\lbrace (\rvx_1,y_1), \cdots, (\rvx_l,y_l) \rbrace$. The prediction of the model is denoted as $f_{\rvw,s_c}(\rvx_p)$. We denote the factor values of $\rvx$ as $v_\rvx$ and the corresponding factor value for factor $e$ as $v_\rvx^{(e)}$. $v_{p}$ is short for $v_{\rvx_p}$. The hidden factor is denoted as $e_h$. We denote the mapping function as $m$, which maps the factor value to the corresponding label, i.e. $y_p=m(v_p^{(e_h)})$. We denote the probability as $\sP$.

\begin{definition}
    \label{def:data_gen}
    The data is generated according to the equation that 
    \begin{equation}
    \sP(\rvx_p,y_p,s_c)=\sum_{m,e_h}\sP(m)\sP(e_h)\sP(\rvx_p,y,s_c|m,e_h), 
    \end{equation}
    where $\sP(m),\sP(e_h)$ are manually setted distributions and $\sP(\rvx,y,s_c|m,e_h)$ is a fixed distribution.  $\sP(m),\sP(e_h)$ are uniform distributions over all possible values by default. In this paper, we rely on changing $\sP(m)$ to obtain the tasks with different complexity.
\end{definition}

\subsection{Analysis}

To successfully tackle the in-context learning task, the network is required to discern the values of the six factors present in the prompt image, which relate to the weights component, as well as accurately determine the appropriate hidden factor for output based on the given contexts, which pertains to the context component.
We break down the distribution $\sP(y_p|\rvx_p,s_c)$ as follows.

\begin{proposition}
\label{prop:decomposite}
The probability of $\sP(y_p|\rvx_p,s_c)$ can be decomposite as:
\begin{equation}
\begin{aligned}
    &\sP(y_p|\rvx_p,s_c)=
    \\ & \sum_{v_p,m,e_h}\underbrace{\sP(y_p|v_p,m,e_h)}_{\text{Properties of Task}}\underbrace{\sP(v_p|\rvx_p)}_{\text{weights}}\underbrace{\sP(e_h|s_c,m)\sP(m|s_c)}_{\text{context}},
\end{aligned}
\end{equation}
where  $\sP(v_p|\rvx_p)$ is weights related information, and $\sP(e_h|s_c,m)\sP(m|s_c)$ is context related information. $\sP(y_p|v_p,m,e_h)$ is related for the properties of task, and we have $\sP(y_p|v_p,m,e_h)=1$ if $m(v_p^{e_h})=y_p$ else $\sP(y_p|v_p,m,e_h)=0$.
\end{proposition}

Ideally, if we can approximate the distribution $\sP(v_p|\rvx_p)$ and the distribution $\sP(e_h|s_c,m)\sP(m|s_c)$, we can obtain the distribution $\sP(y_p|\rvx_p,s_c)$. Recall that we denote $g_{weights}(\rvx_p)$ as weights component and $g_{context}(s_c)$ as context component. Based on the decomposited results of $\sP(y_p|\rvx_p,s_c)$, we define our expectation for the components to be good:

\begin{definition}
    \label{def:good}
    If $f(\cdot)$ has a good weights component in its representation, for any $\rvx$, we can infer  $\sP(v_{\rvx}|\rvx)$ from $g_{weights}(\rvx)$, and if it has good context component, we can infer $\sP(e_h|s_c,m)$ from $g_{context}(s_c)$.
\end{definition}
% Because the learning algorithm has a 

\subsection{Evaluation framework}
In practice, since we cannot find the specific form of $g_{weights}$ and $g_{context}$, such that $f_{\rvw,s_c}(x_p)=g_{comb}(g_{weights}(\rvx_p),g_{context}(s_c))$. We leverage the probe method to measure the goodness of the components in the inner representation of the Transformer as defined in Definition \ref{def:good}. We choose the layers in the Transformer that can produce the best probe results. We give the details of the probing framework in the Appendix \ref{app:probe_model}.

\textbf{Probing methods and metrics} \quad We use three metrics here. \textbf{weights comp.  score}: accuracy of the probe model to predict $v_{\rvx}$ given $\rvx$, where ``comp." is short for component. \textbf{context comp. score}: the accuracy of the probe model to predict $e_h$ given $s_c,\rvx_p$. \textbf{Accuracy}: We evaluate the prediction accuracy of the prompt example as the metric to evaluate the in-context learning performance. We give 39 context examples when evaluate context comp. score and accuracy. 

All the results given in this paper are evaluated in the test set by default, as the performance of the test set is more aligned with the model's performance on real situations.  

\subsection{Model, training and dataset detail}

To simulate the auto-regression framework, we calculate the loss for the sequence $s=\{(\rvx_1,y_1),\ldots,(\rvx_L,y_L)\}$ as:
\begin{equation}
\label{eq:seq_loss}
    \mathcal{L}(\theta,s)=\frac{1}{L}\sum_{i=1}^L \boldsymbol{l}( f_{\rvw,s^{(i-1)}}(\rvx_i),y_i ),
\end{equation}
where $s^{(j)}\triangleq \lbrace (\rvx_1,y_1),\cdots, (\rvx_j,y_j) \rbrace$, $\boldsymbol{l}$ denotes the loss function. $\rvx$ will be tokenized by VAE~\citep{kingma2013auto} before being passed to Transformer. Note that the weights of VAE are pretrained and fixed during the whole experiments. The training loss in the dataset $S$, which contains $n$ sequence, is calculated as the average of loss over all training sequences, i.e.,
\begin{equation}
    \mathcal{L}(\theta,S)=\frac{1}{n}\sum_{s\in S} \mathcal{L}(\theta,s).
\end{equation}
We leave the details of the model, dataset, training design and configuration in the Appendix \ref{app:syn}.

\section{Learning plateaus and weights component}
Recall that we define learning plateaus as periods during the learning process in which the model experiences minimal or no improvement in performance on \textbf{test data}. Conversely, the transition process is characterized by a time span in which the model's performance rapidly enhances. Typically, a learning plateau precedes a transition process. When a learning process includes several instances of learning plateaus and transition processes, the terms “learning plateaus” and “transition process” refer by default to the first occurrence of each. 
\subsection{Controlling of complexity}
We aim to understand when the learning plateaus will happen by controlling the complexity of the tasks. Therefore, we will present the method for controlling the complexity in this part.
Recall from Definition \ref{def:data_gen} that the data is generated following the formula $\sP(\rvx,y,s_c)=\sum_{m,e_h}\sP(m)\sP(e_h)\sP(\rvx,y,s_c|m,e_h)$. By altering the probabilities $\sP(m)$, we can manipulate the resulting dataset.
We first give two baseline configurations:
\begin{itemize}
\item $D_{\fx}$: There exsits a $m_0$ such that $\sP(m_0)=1$ and for all $m\neq m_0$, we have $\sP(m)=0$.
\item $D_{\rnd}$: $\sP(m)$ is a uniform distribution over all possible values.
\end{itemize}
Clearly, introducing greater randomness in the selection of $m$ increases the complexity of the problem. We quantify this complexity using entropy, defined as $\sH \triangleq \sum_{m} -\sP(m) \log \sP(m)$. It is noted that $\sH(D_{\fx})=0$. Since our primary objective is to adjust the complexity of the in-context learning task rather than to examine the effects of various distributions of $\sP(m)$, we simplify $\sP(m)$ to a uniform distribution. The entropy is then managed by varying the size of the support set, that is, the number of possible mapping functions.

When considering two distinct data configurations $D_1$ and $D_2$, the notation $D_1 \Rightarrow D_2$ denotes the performance evaluation of a model on data setting $D_2$ after it has been trained on data setting $D_1$. Unless specified otherwise, we assume that the model is both trained and tested on the same data configuration.

\subsection{Learning plateaus of task with different complexity}

\begin{figure}
    \centering
    \includegraphics[width=0.48 \textwidth]{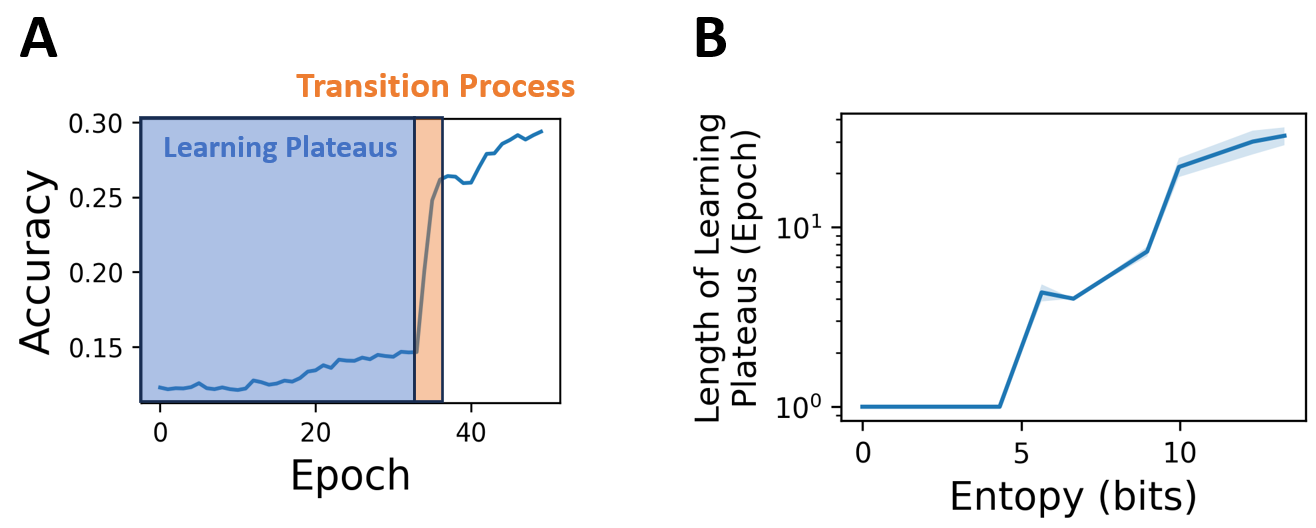}
    \caption{\textbf{Learning plateus.} \textbf{A.} We reproduce the learning plateaus and transition pattern in our synthetic task, similar to Fig~\ref{fig:examples}B. \textbf{B.} The length of learning plateaus increase with the complexity of the task measured by entropy of $\sP(m)$.}
    \label{fig:learning_pateaus}
\end{figure}

\textbf{Learning Plateuas and Transition Point} \quad The primary issue at hand involves learning plateaus and transition patterns. To begin with, we investigate the capability of the synthetic task to mimic the specific pattern depicted in Fig.~\ref{fig:examples}B. Our scrutiny is directed toward the learning trajectory of the synthetic task. As demonstrated in Fig.~\ref{fig:learning_pateaus}A, it becomes evident that the synthetic task is successful in replicating the plateaus and transition pattern observed in the actual task.

\textbf{The length of learning plateaus increases when increases the complexity of task} \quad
We delved into the relationship between task complexity and the duration of learning plateaus. We employed the entropy of the mapping function $m$'s distribution as an indicator of task complexity. To pinpoint the transition process, we identified the first epoch at which the model achieved a test accuracy greater than 0.17. This threshold was selected because the model's accuracy remains below 0.17 before reaching the transition process and rises above this thereafter. As anticipated, more complex tasks necessitate longer learning plateaus (Fig.~\ref{fig:learning_pateaus}B). However, the relationship between the length of the plateaus and the entropy is not linear. With every unit increase in entropy, the extension of the learning plateau is marginal when the entropy is either low or relatively high; conversely, the growth in plateau length is more pronounced at intermediate levels of entropy.
\begin{figure*}
    \centering
    \includegraphics[width=0.93 \textwidth]{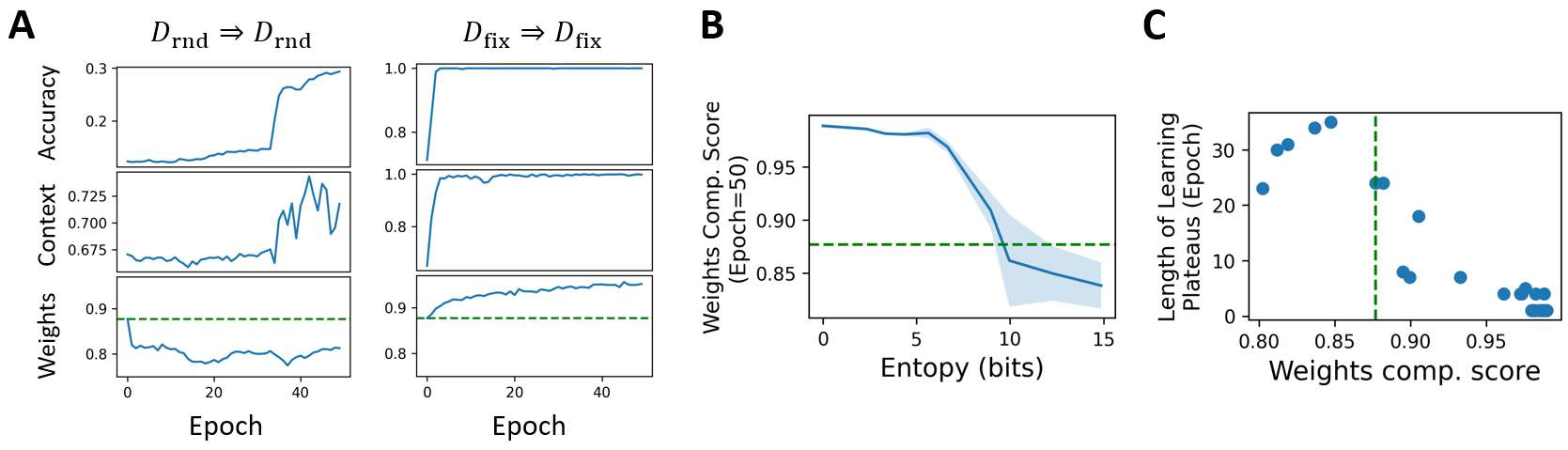}
    \caption{Weights component and learning plateaus. \textbf{A.} The weights component score is increasing under $D_\fx \Rightarrow D_\fx$, while the weights component score is descreasing under the $D_\rnd \Rightarrow D_\rnd$ setting. Note that the “Weights” and “Context” in the figure are short for weights comp. score and context comp. score respectively. \textbf{B.} The weights component score after 50 epoch training decreases when increasing the complexity of the task. The dashed green line indicates the weights component score at the initialization point. \textbf{C.} The weights component score at 50 epoch negative correlates with the length of learning plateaus. The dashed green line indicates the weights component score at the initialization point.}
    \label{fig:weights}
\end{figure*}

\subsection{Dysfunction of weights component}

In the preceding section, we analyzed the plateaus-and-transition pattern in in-context learning concerning task complexity. In this section, our objective is to delve deeper into the role that internal mechanisms—specifically, the quality of the weights and context components—have in influencing the plateaus-and-transition pattern.

\textbf{Confusing pattern of weights component} \quad We executed the task under two conditions: $D_\fx \Rightarrow D_\fx$ and $D_\rnd \Rightarrow D_\rnd$, with the outcomes presented in Fig.~\ref{fig:weights}A. As expected, in the simpler scenario of $D_\fx \Rightarrow D_\fx$, both the context and weights components exhibit improvement throughout the learning process, leading to enhanced in-context learning performance. However, in the more challenging setting of $D_\rnd \Rightarrow D_\rnd$, the weights component deteriorates over time, with its score remaining below the initial value for the entire duration of the learning process. This differs from the context component, which improves with the rise in in-context learning performance. We refer to the situation where the weights component score falls below the starting value as a \textbf{dysfunction} of the weights component. This outcome is intriguing because training has no effect in improving the weights component. To gain better insight into this phenomenon, we carried out additional experiments across varying levels of task complexity. The model was trained for 50 epochs on these tasks, and we monitored the weights component score post-training. We found that the weights component score gradually declines as the entropy increases, eventually stabilizing at around 0.8, as shown in Fig.~\ref{fig:weights}B.

\textbf{Weight component degradation is linked to duration learning plateaus.} \quad Our primary concern is the duration of learning plateaus, and we seek to comprehend its connection to the in-weights component. To investigate this, we graphed the relationship between learning plateau length and weights component score after 50 epochs, as shown in Fig.~\ref{fig:weights}C. Our analysis reveals an approximately linear correlation between the weights component score and the learning plateau duration. Short learning plateaus occur when there is a significant improvement in the in-weights component score. Conversely, long learning plateaus arise when the weights component is dysfunctioning or on the edge of dysfunction, that is when the weights component score is at or below the initial value.

\textbf{Why the weights component is related to learning plateaus.} \quad We are attempting to comprehend this phenomenon, yet theoretically analyzing the learning process of a multilayer transformer on intricate data proves to be a formidable challenge. Recent theoretical studies~\citep{tian2023scan,deora2023optimization,huang2023context} focus on the learning dynamics of Transformers with one or two layers using simple datasets. Given these limitations, we propose a more attainable, albeit weaker, construction analysis in Appendix \ref{sec:contruction_analysis}. Our approach is grounded in the notion that if a model with a good weights component can enhance its in-context learning capabilities with just a few additional parameters, then the weights component must be pivotal for achieving in-context learning prowess. To test this idea, we hypothesize that the in-weights component has been perfectly learned within a specific layer of the Transformer. Our findings reveal that by adding at most three extra Transformer layers specifically tailored for processing contextual information, the model demonstrates significant in-context learning performance. These outcomes suggest that in-context learning abilities are more readily achieved when the weights component is well-optimized.

\section{Breaking Through Learning Plateaus}
The previous section examined the mechanisms behind learning plateaus and identified their connection to the weights component. In this section, our goal is to investigate \textit{whether we can shorten
the learning plateaus or improve the performance increasing for each transition process without scaling models}. We have chosen the $D_\rnd$ configure as the configuration of the test set by default, as it has demonstrated pronounced learning plateau behavior, and this particular scenario is known to be more effective in evaluating in-context learning capabilities, as discussed by~\citet{wei2023larger,min2022rethinking}.
% \begin{figure}
%     \centering
%     \includegraphics[width=1 \textwidth]{images/different_settings.pdf}
%     \caption{\textbf{Performance on $D_\rnd$ with different training settings.} \textbf{A}: $D_\rnd \Rightarrow D_\rnd$ only improve the context component. \textbf{B}: $D_\fx \Rightarrow D_\rnd$ only improve the weights component. \textbf{CD}:  An improved weights component can speed up the learning process of the context component. The green dashed line marks the point when switch from $D_{\fx}$ to $D_{\rnd}$}
%     \label{fig:diff_settings}
% \end{figure}
\begin{figure*}
    \centering
    \includegraphics[width=0.9 \textwidth]{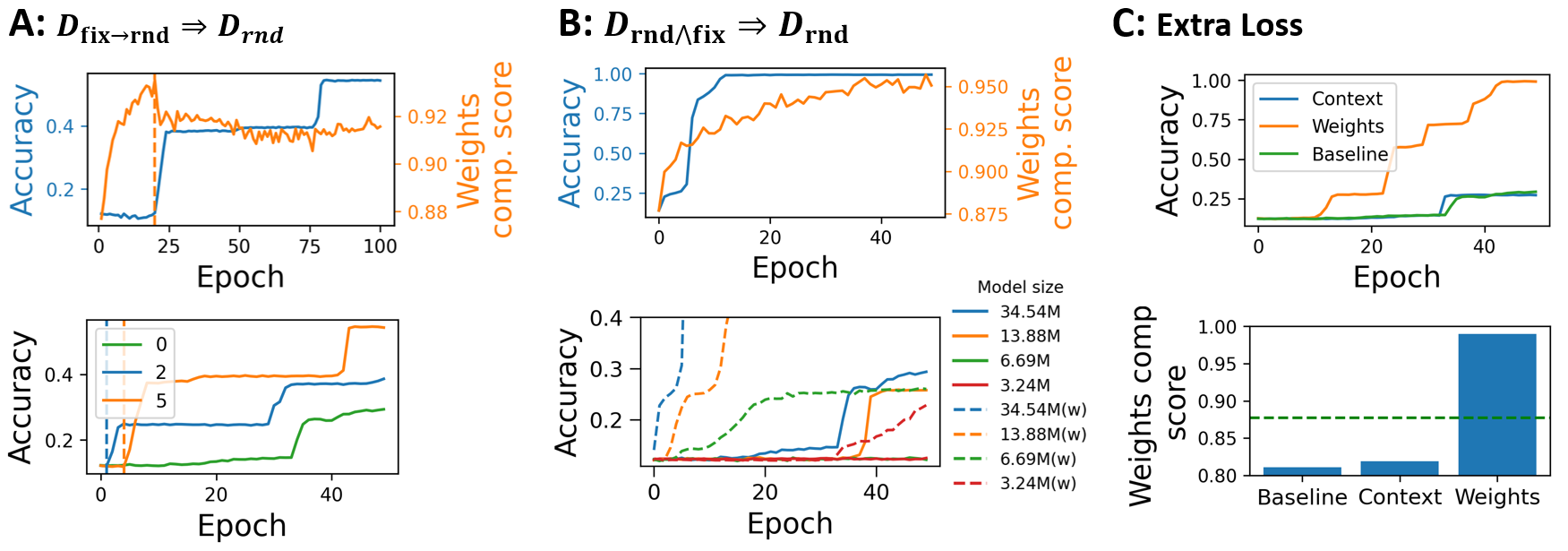}
    \caption{Three methods are proposed to assist in overcoming learning plateaus. \textbf{A: Effective of the warm-up method.} \textit{Top:} Employing $D_{\fx}$ as a warm-up for the Transformer significantly mitigates learning plateaus. The dashed line indicates the transition point from $D_{\fx}$ to $D_{\rnd}$.  \textit{Bottom:}  We execute the transition from $D_{\fx \to \rnd} \Rightarrow D_{\rnd}$ at various switching points. The curve labeled "2" signifies the switch from $D_{\fx}$ to $D_{\rnd}$ at epoch 2. The curve labeled "0" serves as the baseline, that is, $D_{\rnd} \Rightarrow D_{\rnd}$. The dashed lines highlight the respective switching points. \textbf{B: Combining $D_{\fx}$ and $D_{\rnd}$.} \textit{Top:} Mixed training substantially improves the weights component score during the learning process and eliminates learning plateaus. \textit{Bottom:} Boosting the weights component can promote the development of in-context learning capabilities in smaller models. The dashed line depicts the task configuration $D_{\fx \land \rnd} \Rightarrow D_{\rnd}$, while the solid line represents the $D_{\rnd}\Rightarrow D_{\rnd}$ setting. \textbf{C: Extra Loss.} \textit{Top:}  Incorporating a weights loss can significantly enhance learning, whereas adding context loss does not have a noticeable impact. The baseline is $D_{\rnd} \Rightarrow D_{\rnd}$. \textit{Bottom:} With the weights loss, the Transformer can attain a commendable weights component score after 50 epochs of training. The green dashed line indicates the weights comp. score at the initialization point.}
    \label{fig:break}
\end{figure*}

\subsection{Weights warm-up method}

Previous experiments demonstrate that the model better learns the weight component more effectively in the $D_{\fx}$ setting. A straightforward intuition is whether we can use $D_{\fx}$ to improve the weights component in $D_{\rnd}$ test set. We give a further analysis of the relation between $D_\rnd$ and $D_\fx$:

\textbf{Increasing weights component with $D_{\fx}$:} Recall that we have $\sP(y_p|\rvx_p,s_c)\sim \sP(v_p|\rvx_p)\sP(e_h|s_c,m)\sP(m|s_c)$. Under $D_{\fx}$ setting, we only has one mapping function $m_0$. Consequently, our model can readily learn that $\sP(m_0|s_c)=1$. This simplification allows the model to concentrate on mastering $\sP(v_\rvx|\rvx)$. Therefore, it is anticipated that the model will develop a more refined weights component under this setting.  \textbf{Knowledge transfering between $D_{\fx}$ and $D_{\rnd}$ settings:} The knowledge of $\sP(v_\rvx|\rvx)$ is shared between these two settings. The reason is that $\sP(v_\rvx|\rvx)$ is unrelated to the mapping function $m$ and the context $s_c$. 

Based on this analysis, we propose the following data configuration to improve the weights component on $D_\fx$ before training on $D_\rnd$. This setting is denoted as $D_{\fx \to \rnd}$, which means that we initially train the model on $D_{\fx}$ for a specific epoch (weights warm-up), and then, we train the model on $D_{\rnd}$. We have the following discovery based on the experiments on $D_{\fx \to \rnd} \Rightarrow D_{\rnd}$.

\textbf{Weights warm-up helps to mitigate the learning plateaus.} From Fig.~\ref{fig:break}A Top, we observe a notable enhancement in the weights component during the initial “warm-up” phase. Subsequent to this phase, there is a swift improvement in the accuracy of the in-context learning. Nonetheless, a subsequent decline in the weights component score post-warm-up is apparent. This further verify that the Transformer can not learn a good weights component under $D_\rnd$ settings. Further investigation is conducted by training the Transformer with various “warm-up” durations under the $D_\fx$ setting. As shown in Figure\ref{fig:break}A Bottom, training the Transformer for a brief number of epochs effectively eliminates the initial learning plateaus. Furthermore, this “warm-up” phase contributes to a reduction in the duration of the subsequent learning plateaus. The application of the “warm-up” technique yields significantly better results compared to those without it, particularly after a training duration of 50 epochs.

\subsection{mixed training method}
\label{subsec:joint_effect}

In this section, we try to improve the weights component during the whole training process by designing the $D_{\fx \land \rnd} \Rightarrow D_\rnd$ setting.  In $D_{\fx \land \rnd}$ setting, half training data comes from $D_\fx$ setting, and half training data comes from $D_\rnd$ setting. For a fair comparison, the total number of training data is the same as the training set that uses $D_\rnd$ or $D_\fx$ setting only.

\textbf{The mixed training method significantly boosts the learning process.} \quad As anticipated, both the weights component and the task accuracy exhibit improvements throughout the learning process, as depicted in Fig.\ref{fig:break}B Top. And the results do not show a pronounced learning plateau. This suggests that concurrently enhancing the weights component is beneficial. In Fig.\ref{fig:break}B Bottom, we compare models of various sizes trained on the $D_{\rnd} \Rightarrow D_{\rnd}$ setting (represented by the solid line) against those trained on the $D_{\rnd \land \fx} \Rightarrow D_{\rnd}$ setting (depicted with a dashed line). It is observed that a model of size 6.69M trained on $D_{\fx \land \rnd}$, after 20 epochs of training, can attain an accuracy comparable to a model of size 13.88M trained on $D_{\rnd}$ that has undergone 50 epochs of training. These findings suggest that overcoming learning plateaus can lead to a reduction in the computational resources required for training.

\subsection{Extra loss method}

The previous methods require another data setting $D_{\fx}$ to improve the weights component. Here, we consider another alternate method by providing an extra supervision signal to improve the weights component when we cannot find the $D_{\fx}$ setting.  We consider two extra loss settings. The $f'_{\rvw',s_c}(x)$ is denote as the subnetwork of $f_{\rvw,s_c}(x)$ without the output classifier. We denote the $cls_e$ as the classifier for the factor value of factor $e$.
The weights loss is defined as
\begin{equation}
    \label{eq:extra_weights}
    \mathcal{L}_{w}(\theta,s)=\frac{1}{L|E|}\sum_{e\in E}\sum_{i=1}^L \boldsymbol{l}( cls_e (f'_{\rvw',s^{(i-1)}}(\rvx_i)),v^{(e)}_{\rvx})
\end{equation}
The context loss is defined as 
\begin{equation}
    \mathcal{L}_{c}(\theta,s)=\frac{1}{L}\sum_{i=1}^L \boldsymbol{l}( cls_{e_h} (f'_{\rvw',s^{(i-1)}}(\rvx_i)),e_h^{(s)}),
\end{equation}
where $cls_{e_h}$ is the classifier to predict hidden factor, and $e_h^{(s)}$ is the hidden factor for sequence $s$.
Then, the original loss function of Equation \ref{eq:seq_loss} is modified into $\mathcal{L}(\theta,s)+\lambda \mathcal{L}_{c}(\theta,s)$ or $\mathcal{L}(\theta,s)+\lambda \mathcal{L}_{w}(\theta,s)$. The $\lambda$ is chosen as $0.1$ in our experiments.

\textbf{Add weights loss speedup the learning process while adding context loss fails.} \quad Fig.~\ref{fig:break}C Top reveals that incorporating an additional weights loss significantly aids the Transformer model in overcoming learning plateaus. In contrast, adding an extra context loss yields only a marginal benefit. These outcomes further substantiate our hypothesis that enhancing the weights component is crucial for breaking through learning plateaus, rather than concentrating on the context component. As shown in Fig.~\ref{fig:break}C Bottom, we note a marked improvement in the weights component when an extra weights loss.

\subsection{Further Exploration: Simple Functions Tasks}

In this section, we give further exploration of the relation between the weights component and the learning plateaus on the simple function tasks~\citep{garg2022can} and further understand the role of weights component in learning plateaus.

\textbf{Two Roles of Weights Component} \quad Initially, I'd like to emphasize that the weights component serves two primary functions:
1) It takes the examples $x_i$ into the internal working of the Transformer.
2) It aims to find a better representation for $x_i$.  Considering this, the dysfunction of the weights component can stem from these two sources correspondingly.
In the previous analysis of the Shape3D task, the weights component draws from two sources. Since simple function tasks lack representation learning, the dysfunction of the weights component might primarily be attributed to the first reason.

\textbf{Date generation process} \quad In this task, we generate a dataset with $n$ sequences. The generation process for each sequence is:
\begin{itemize}
    \item Initially, we sample a vector $\rvw$ from the Gaussian distribution $\mathcal{N} (0,\mathbf{I}_d)$ in $\mathbb{R}^{d}$. Each component of $\rvw$ is drawn independently.
    \item Subsequently, we sample $\lbrace \rvx_1,\rvx_2,……,\rvx_L \rbrace$, where each $\rvx_i \in \mathbb{R}^d$ is drawn from a Gaussian distribution $\mathcal{N}(0,\boldsymbol{I}_d)$. 
    \item The labels are then determined using the formula $y_i=sign (\rvw^T \rvx_i)$. Finally, we obtain a sequence $s=\lbrace (\rvx_1,y_1), \cdots, (\rvx_L,y_L) \rbrace$.
\end{itemize}

\textbf{Exploring Framework} \quad The training methodology aligns with that described in the paper~\citep{garg2022can}. We'll modify the dimension of $\rvx_i$, i.e. the value of $d$ to modulate the difficulty of the task. The evaluation of the weights component follows the same process outlined in Appendix \ref{app:probe_model}, but with the utilization of the metric $MSE_p=\mathbb{E}_{\rvx_i}{\Vert \tilde{\rvx}_i-\rvx_i\Vert}$ to assess its performance, where $\tilde{\rvx}$ is the prediction of the probe model. We use the Transformer with 6 layers and we probe at the layer 3. It's worth noting that unlike the weight component score, where higher is better,  lower values of $MSE_p$ indicate better weights component.
\begin{figure}
    \centering
    \includegraphics[width=0.45 \textwidth]{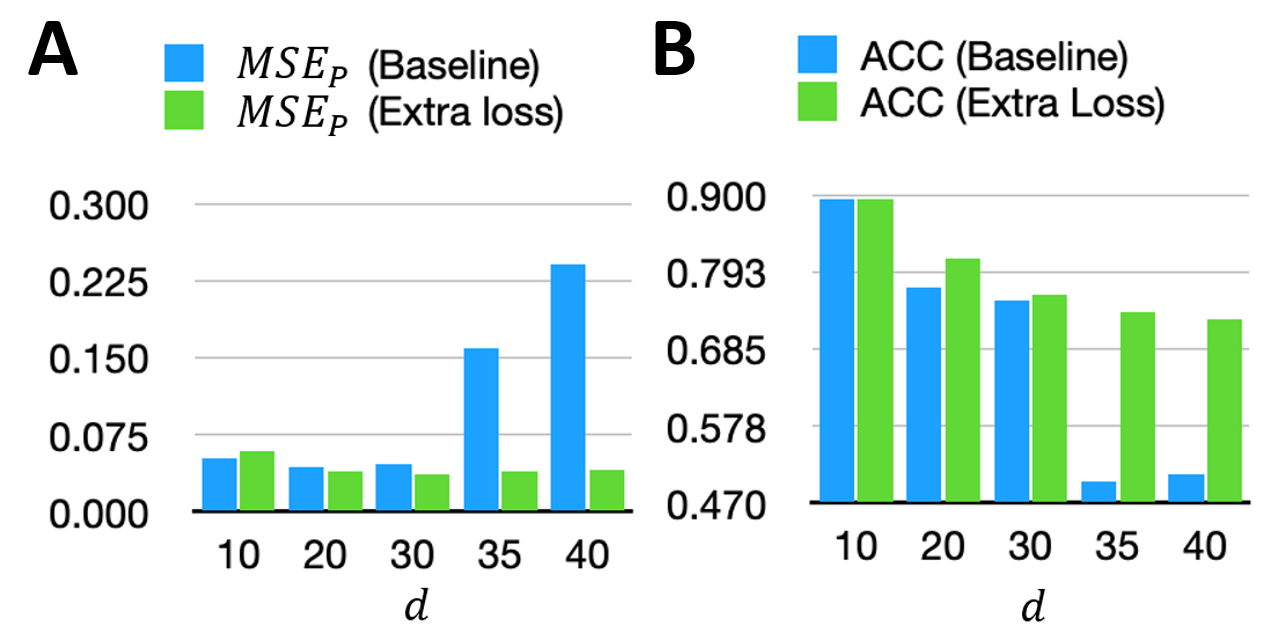}
    \caption{\textbf{Experiments results on simple function tasks after 50 epochs.} \textbf{A:} the dysfunction of the weights component happens when $d>30$. \textbf{B:} the effect of Extra Loss technique is significant only when the dysfunction of weights component happens.}
    \label{fig:simple_functions}
\end{figure}

\textbf{Results} \quad The $MSE_p$ and test accuracy are given in Figure \ref{fig:simple_functions}. We find that: 1) Dysfunction in the weights component is evident in the SimpleFunction dataset when $d>30$.
The Transformer exhibits poor performance in these situations. 2) the effect of the Extra Loss technique is significant only when the dysfunction of weights component happens. The effect is incremental when the Transformer doesn't experience a significant weights dysfunction. This further verifies the causal relation between the weights component and learning plateaus.

\subsection{Beyond Synthetic Task}
\label{sec:beyond}
\begin{figure*}[h]
    \centering
    \includegraphics[width=0.85 \textwidth]{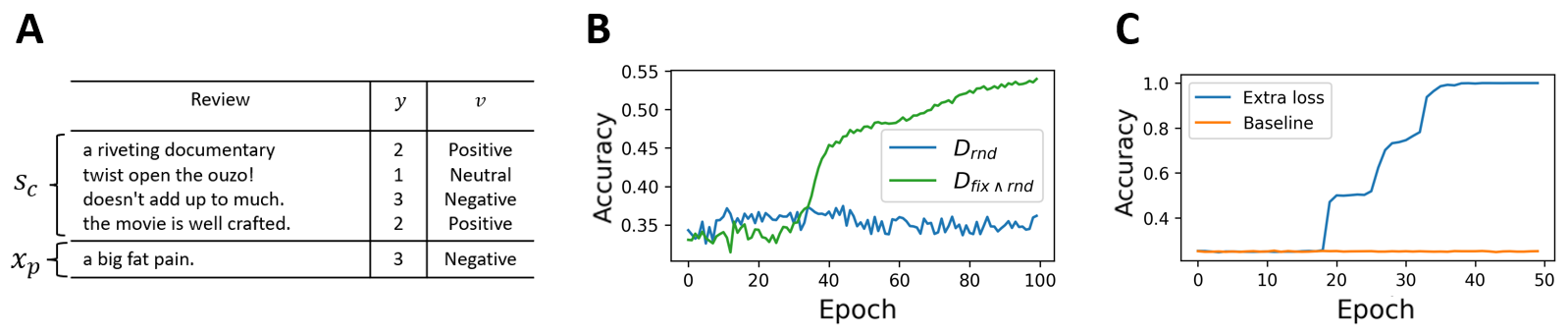}
    \caption{\textbf{Experiments on natural language tasks.} \textbf{A:} Example of dataset SST-ICL task. The example label $y$ is obtained from $v$ by map function $m$, i.e. $y=m(v)$. \textbf{B:} Results on SST-ICL task. We explore the $D_{\fx \land \rnd} \Rightarrow D_{\rnd}$ and $D_{\rnd} \Rightarrow D_{\rnd}$ settings. The dashed line denotes the time when we transit from $D_{\fx}$ to $D_{\rnd}$. \textbf{C:} Results on WordSelection Task. Adding extra weights loss has a significant effect in shortening the learning plateaus.
    }
    \label{fig:text}
\end{figure*}
To ensure that the proposed methods are broadly applicable, this section examines their adaptability to natural language processing (NLP) tasks.

\textbf{SST-ICL task} \quad The task employs the SST dataset as outlined by \citet{socher2013recursive}. An illustration of the task is given in Fig.~\ref{fig:text}A. Comprehensive details regarding the dataset structure, training methods, and model configurations are available in Appendix \ref{sec:nlp_app}. We assess the results for both the $D_{\fx \land \rnd} \Rightarrow D_{\rnd}$. It is noted that improving the weights component through training on $D_{\fx}$ similarly aids in-context learning within this domain. Nevertheless, it should be acknowledged that variations exist between the outcomes of the SST-ICL task and those of the synthetic task; specifically, $D_{\fx}$ settings worse the performance when trained on small epochs, which is evident in the Figure~\ref{fig:text} that the $D_{\fx \land \rnd}$ performs worse than $D_{\rnd}$ when training epoch is less than 30. The key reason to this divergence is that the weights component is more difficult to improve compared that in the synthetic one. As a result, more epochs of training are needed to improve the weights component so as to make it come into effect. This explanation is evident by the result in Figure \ref{fig:text} that the model trained on $D_{\fx \land \fx}$ setting significantly outperforms that trained on $D_{\rnd}$ settings after the 35 epochs.

\textbf{WordSelection task} \quad We have devised an additional task named the WordSelection Task, which requires selecting a single word from a group of four options. For example, given the input “hello information learning art → learning”, the task for the model is to identify “learning” as the correct choice from the provided set. The model must infer the correct answer by considering context examples. We offer five such in-context examples, all selecting the words at the same position as the prompt example.
Creating a specific $D_{\fx}$ setting to improve the weights component poses a challenge in this task. Consequently, we opt for the implementation of the extra weights loss technique. We define the factor as the position and the factor values of a factor is the corresponding word in this position. For example, with $\rvx$ being “hello information learning art”, we define $v_{\rvx}^{(e_1)}$ as “hello”, $v_{\rvx}^{(e_2)}$ as “information”, $v_{\rvx}^{(e_3)}$ as “learning” and $v_{\rvx}^{(e_4)}$ as “art”. The extra weights loss is then applied in a manner consistent with what is described in Eq. (\ref{eq:extra_weights}).
Our experimental findings align with those from the synthetic task, indicating that the advantage of integrating an extra weights loss is not limited to synthetic environments. The results highlight that adding extra weights loss can be a beneficial strategy across different task types.

\section{Discussion}

\textbf{Q1: What are the reasons and intuitions behind the dysfunction of weights?} \quad The observed dysfunction in the Transformer's performance may stem from the confluence of two crucial factors: 1) The problem-solving process relies on the information derived from both context examples and query examples.  2) The contextual information is hard to obtain. 
In situation 1, where both contextual and weight information are critical, the Transformer is compelled to harness both effectively. Yet, under condition 2, where the contextual information proves difficult to learn, the Transformer finds itself unable to fully extract this vital knowledge. Consequently, it lacks the incentive to refine the weight component at this stage, as doing so would not yield appreciable improvements in overall performance. This phenomenon aligns with our observation in Figure 4C that the weight dysfunction occurs predominantly when the in-context tasks exhibit a high degree of complexity, as measured by the high entropy of the distribution of $m$.
As a result, we witness the manifestation of weight dysfunction, characterized by the stagnation or regression of the weight component during the learning process, despite the Transformer's ongoing attempts to adapt and solve the given problem.

\textbf{Q2: Whether the decomposition of weights and context components is general?} \quad The conceptual decomposition is general based on the following reasons:

1) Firstly, the decomposition of the weights component and the context component is conceptual instead of physical and only for analysis purposes. The decomposition stems from the understanding that the in-context learning task necessitates information from both its context samples and the query example. 

2) Secondly, our evaluation of the weights component and context component is not based on the physical decomposition of these components. We employ complex probe methods (see Section A.3) to analyze these two components because we have only conceptually decomposed them.

3) We tested our method across various scenarios, including the Shape3D task, SST-ICL task, Word Selection Task, and SimpleFunction task, which consistently validated the efficacy of our approach.

\section{Conclusion}
This paper establishes a connection between the weights component and learning plateaus. Building on this connection, the paper proposes three strategies to overcome these plateaus. These strategies have proven to be effective in both synthetic and natural language tasks.

\section{Limitation}
In this paper, we mainly focus on understanding of the learning plateaus of in-context learning with Transformer, where the Transformer requires the information from the context examples and the query example to make a prediction. Therefore, our method cannot explain the learning plateaus phenomenon outside this score. There are some works~\cite{nanda2023progress,power2022grokking} that discuss the learning plateaus phenomenon in supervised learning. Our work fails to explain these phenomena.

\section*{Acknowledgements}
Jingwen Fu, Tao Yang, and Nanning Zheng are supported by the National Natural Science Foundation of China (Grant No. 62088102).

\section*{Impact Statement}
The goal of this paper is to understand the mechanism of learning plateaus in Transformers and find a method to avoid learning plateaus. Our work doesn't have a direct influence on society. However, future works based on our work may influence society but it is unpredictable currently.


\bibliography{example_paper}
\bibliographystyle{icml2024}

%%%%%%%%%%%%%%%%%%%%%%%%%%%%%%%%%%%%%%%%%%%%%%%%%%%%%%%%%%%%%%%%%%%%%%%%%%%%%%%
%%%%%%%%%%%%%%%%%%%%%%%%%%%%%%%%%%%%%%%%%%%%%%%%%%%%%%%%%%%%%%%%%%%%%%%%%%%%%%%
% APPENDIX
%%%%%%%%%%%%%%%%%%%%%%%%%%%%%%%%%%%%%%%%%%%%%%%%%%%%%%%%%%%%%%%%%%%%%%%%%%%%%%%
%%%%%%%%%%%%%%%%%%%%%%%%%%%%%%%%%%%%%%%%%%%%%%%%%%%%%%%%%%%%%%%%%%%%%%%%%%%%%%%
\newpage
\appendix
\onecolumn
\section{Detail of experiments on sythetic tasks}
\label{app:syn}
\subsection{3DShape} 
3dshapes\footnote{https://github.com/deepmind/3d-shapes}~\citep{kim2018disentangling} is a dataset of 3D shapes procedurally generated from 6 ground truth independent latent factors. These factors are floor colour, wall colour, object colour, scale, shape and orientation. All possible combinations of these latents are present exactly once, generating N = 480000 total images. Latent factor values including:
1) floor hue: 10 values linearly spaced in [0, 1],
2) wall hue: 10 values linearly spaced in [0, 1],
3) object hue: 10 values linearly spaced in [0, 1],
4) scale: 8 values linearly spaced in [0, 1],
5) shape: 4 values in [0, 1, 2, 3], and
6) orientation: 15 values linearly spaced in [-30, 30].

% \section{More details about experiments}
% \label{sec:exp_detail}

% \subsection{Model and training configure}

\paragraph{Shape3D dataset} We employ the Adam optimizer \citep{kingma2014adam} and mini-batch training to optimize the loss function $\mathcal{L}(\theta,S)$. Here, we use cross-entropy as the loss function. We utilize a batch size of 128 and set the learning rate to 0.0001. For training purposes, we use $10^5$ sequences and, for evaluation, $4 \times 10^4$ sequences. There is no overlap between images in the training sequences and those in the evaluation sequences.

\subsection{Architecture Detail}
We employ a pre-trained Variational Autoencoder (VAE) to transform images into tokens. The encoder of the VAE comprises seven convolutional layers with ReLU activations, followed by three linear layers two ReLU activations. Conversely, the VAE's decoder is structured with three linear layers and two ReLU activations, which precede a sequence of seven convolutional layers also with ReLU activations. The resulting latent representation serves as the token representation of the input image. The label embedding is learned during the in-context training process. Our study's primary goal is to explore the characteristics of in-context learning. To this end, we utilize a causal Transformer architecture that restricts each token to only interact with preceding tokens. Specifically, the default configuration of the Transformer, denoted as $f$, includes 12 layers, 4 attention heads, and an embedding dimension of 128. As depicted in the lower section of Figure~\ref{fig:break}B, we experiment with varying model sizes. Detailed configurations for these model sizes can be found in Table \ref{tab:config_model}.

\begin{table}[h]
    \centering
    \caption{Model configure with different size in Fig \ref{fig:break}B Bottom}
    \begin{tabular}{c|c|c|c|c}
    \hline
    Number of Layers     & 12 & 6& 6 &3 \\
    Attention head     & 4 & 4& 2 &2\\
    Embedding size     & 128 &64 &32 &16\\
    Model Size     & 34.53& 13.88 & 6.69 & 3.24\\
    \hline
    \end{tabular}
    \label{tab:config_model}
\end{table}

\subsection{More detail about probe framework}
\label{app:probe_model}
\begin{figure}[h]
    \centering
    \includegraphics[width=0.8  \textwidth]{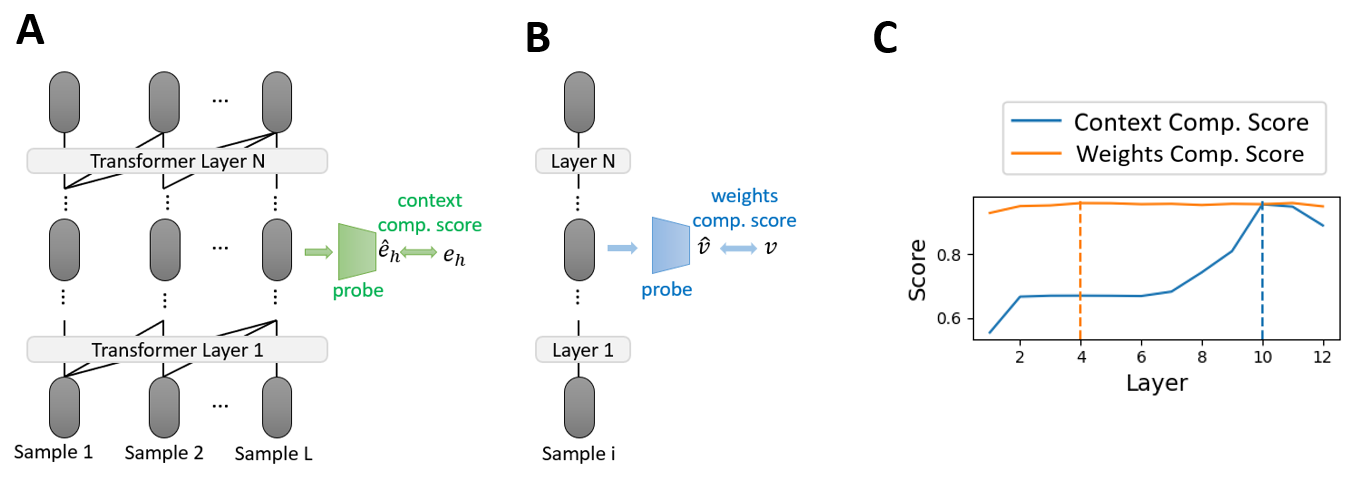}
    \caption{Probe method. \textbf{AB:} illustration of the probe method.  \textbf{C:} The weights and in-context score when we probe at different layers. We choose the $D_{\fx \land \rnd} \Rightarrow D_\rnd$ settings. The dashlines marked the chosen layers in the experiments. }
    \label{fig:probe_method}
\end{figure}
% \begin{figure}
%     \centering
%     \includegraphics[width=0.99 \textwidth]{./images/representation_calculation.pdf}
%     \caption{Illustraction of calculation of context comp score(A), weights comp score(B) and in-context learning score(C). 
%     }    \label{fig:representation_calculation}
% \end{figure}
We employ metrics for numerical evaluation of components and in-context learning performance. Since the components are hidded in the representation, we use the probe method \citep{alain2016understanding}. The probe classifier has a single linear layer, with softmax and cross-entropy calculating the loss. \textbf{Because we don't know the specific form of $g_{weights}$ and $g_{context}$, we choose the representation within the Transformer that given the high score for weights and context components for approximating.} We use the linear probe because using more complex probe model doesn't have significant improve in accuracy and the change of the accuracy during the training process is more important than the absolutely value. The probe model is trained utill totally converge. The probe is trained for 2 epoch for context component and 1 epoch for weights component.
The in-weight probe predicts values of six factors of all images, while the in-context probe identifies the hidden factor for each sequence. The details are as follows.

\textbf{context comp. score} \quad Context comp. score measures whether the Transformer can capture the information from all the context examples, i.e., whether the inner representation within Transformer can capture the distribution $\sP(y_p|\rvx_p,s_c)$. Given the test set $S'$, the context comp. score is calculated as $\frac{1}{|S'|} \sum_{s\in S'} \mathbf{1}_{\hat{e}_{h,s} = e_{h,s}}$, where $\mathbf{1}_{\text{expr}}$ is indicator function, $s$ is the sequence in the dataset $S$,  $e_{h,s}$ is the hidden factor for the  sequence $s$, and $\hat{e}_{h,s}$ is the prediction of probe classifier. We use $|\cdot|$ to denote the corresponding size of a set. 

\textbf{weights comp. score} \quad The weights comp. score measures whether the Transformer can learn the information from individual examples, i.e. whether the inner representation of Transformer can capture the distribution $\sP(v_x|\rvx)$. Because the distribution $\sP(v_x|\rvx)$ is unrelated to the context examples $s_c$, we remove the influence of context examples by removing the context examples when evaluate the weights component score. The weights comp score  is calculated as $\frac{1}{|S'|} \sum_{s\in S'} \frac{1}{|s||E|} \sum_{(x,y) \in s} \sum_{e\in E} \mathbf{1}_{\hat{v}^{(e)}_x=v^{(e)}_x}$, where $v^{e}_x$ is factor value of factor $e$ and sample $x$ , $\hat{v}^{(e)}_k$ is the prediction of probe classifier, $s=\{(\rvx_1,y_1),\ldots,(\rvx_L,y_L)\}$ is the sequence in the dataset $S'$ and $E$ is the set of all factors. 

\textbf{Accuracy} \quad We measure the accuracy of the prediction of in-context learning task given a fix number of context example as the measure for the in-context performance. In this paper, we choose the number of context examples as $39$.

% \subsection{Influence of the number of context examples}
\subsection{Dataset split}
\label{subsec:dataset_split}
\paragraph{In-context training}

We first split all the the images in Shape3D into two part: the training image set (80 \%) and the test image set (20 \%). Then, we organize all the training images into $S_{\fx}$, $S_{\rnd}, S_{\fx \land \rnd}$, corresponding to $D_{\fx}$, $D_{\rnd}, D_{\fx \land \rnd}$ settings. $S_{\fx}$, $S_{\rnd}, S_{\fx \land \rnd}$  Test image set are also organized into $S'_{\fx}$, $S'_{\rnd}, S'_{\fx \land \rnd}$. Each of $S_{\fx}$, $S_{\rnd}$, and $ S_{\fx \land \rnd}$ contains $10^5$ sequences.  Each of $S'_{\fx}$, $S'_{\rnd}, S'_{\fx \land \rnd}$ contains $4 \times 10^4$ sequences.

\textbf{Probe model training} \quad If we want to probe a model $f_{\rvw,s_c}(\cdot)$ on setting $D_{\rnd}$ (test setting), we will first train the probe model on $S_{\rnd}$ with $f_{\rvw,s_c}(\cdot)$ and we evaluate the probe model on $S'_{\rnd}$ with $f_{\rvw,s_c}(\cdot)$. The same for $D_\fx$ and $D_{\fx \land \rnd}$ settings.

\section{Detail of experiments on langugae task.}
\label{sec:nlp_app}

\begin{figure}[h]
    \centering
    \includegraphics[width=0.8  \textwidth]{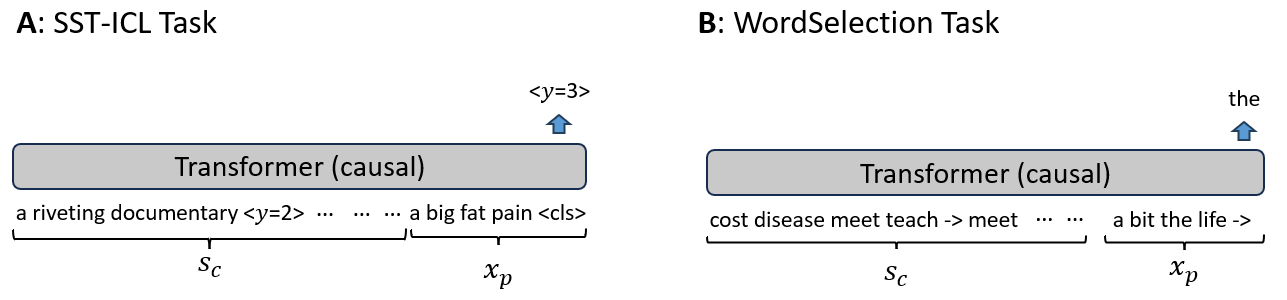}
    \caption{Example of the SST-ICL task and WordSelection task.}
    \label{fig:example_img}
\end{figure}

\subsection{Experiments on Pythia}
We first indroduce the experiment design on the Pythia 13B model, because setup of this experiment is totally different from the other experiments on Section \ref{sec:beyond}.
We leverage the opensouce of the Pythia 13B checkpoints. The Pythia is trained using autoregression framework on the Pile dataset. We evaluate the performance of the Pythia model for each 10k steps. We construct 400 sequence of 4 wordselection task. All the words are randomly sampled from 2000 words. 4 context examples are given for each prompt. The prediction of the Pythia model is processed to ensure that it has same form as the given label.

\subsection{Model Struture}
For both tasks, we use the GPT2 model in this setting. The model is consist of 6 layers, 4 attention heads with 368 embedding size.

\subsection{Dataset Detail}

\textbf{SST-ICL dataset} \quad
The dataset is contructed based on SST \citep{socher2013recursive} datasets. We remove the long review in the datasets and transform the original labels into ``Negative", ``Positive" and ``Neutral". Then, we organize the reviews follow the same way as that in Subsection \ref{subsec:dataset_split}. We produce $10^ 4$ sequence for training and $4 \times 10^3$ for testing. Each sequence contains 5 reviews. We illustrate the example of the dataset in Fig. \ref{fig:text} AB.

\textbf{WordSelection Dataset} \quad We choose 2000 words for the experiments. The 2000 words is organzed into $10^5$ training sequence and $5 \times 10^4$ training sequence. Each sequence contain six examples and five examples served as in-context examples and the rest served as the prompt example.
\label{app:wordselection}

\subsection{Training Detail}
For SST-ICL task and WordSelection task, the models are both trained using AdamW optimizer with learning rate $2e-5$. We choose the batch size as 64.

\section{Other Related works}
\label{sec:other_related_works}

\subsection{In-weights and in-context learning}
Previous studies~\citep{chan2022data,chan2022transformers} examined the relationship between in-weights learning and in-context learning. The division of in-weights learning and the in-context learning process is conceptually similar to our distinction between the influence of the weights and context component. In in-context learning, the Transformer relies on a combination of both its internal weights and the context provided to address a given task. In contrast, for in-weights learning, the model depends exclusively on its weights. \textbf{Nevertheless, there are notable differences between these two notion}s: 1) The internal weights and context components are concerned with capturing information within the Transformer, whereas the focus of in-weights and in-context learning is on how the Transformer tackles a task. In-weights learning is a rename for regular supervised learning. 2) Both weights and context components coexist and play roles within the paradigm of in-context learning. While in-weights learning is distinguished from in-context learning.

\subsection{Evidence of previous works regarding weights and context component are both important for in-context learning}

In this section, we provide evidence about that the in-context and weights components in practice tasks.

\textbf{Intuition 1: Influence of words replacing} \quad A key difference between the weights and context components lies in the susceptibility of the weights component to word substitution. The weights component can be easily disrupted if a word is replaced with a token that was not present during the training phase, as the weights lack information about this new token. On the other hand, if the context examples are rich in information, the meaning of this new token can still be deduced. This mirrors the human ability to infer the meaning of an unknown word based on its context. If word substitution leads to a decline in performance, it suggests that the Transformer's prediction relies heavily on the weights component.

\textbf{Intuition 2:Influence of number of in-context examples} \quad The efficiency of the context component is expected to rise with the inclusion of more context-specific examples, a characteristic not shared by the weights components, which remain unaffected by the addition of in-context examples. Therefore, if performance improves with the integration of more context-specific examples, it would suggest that the Transformer's prediction is heavily influenced by the context component.

\textbf{Intuition 3: Zero-shot performance} \quad The zero-shot performance can directly indicate the effectiveness of the weights component. This is because no in-context examples are provided in this scenario, reducing the problem to a traditional supervised one

Based on the intuitions above, we collect the related experiments in practice paper.

1. \citet{min2022rethinking} discovered that (1) performance can be improved by increasing the number of in-context examples. (2) Changing the labels of in-context examples does not influence the predicted label. The first discovery indicates that the prediction relies on the context components. The second discovery suggests that the Transformer uses the weights component for label prediction, given that there is no observed change when the labels of in-context examples are altered.

2. \citet{brown2020language} found that larger models are increasingly effective at utilizing in-context information. This suggests that in real-world scenarios, the efficiency of the context component improves with the enlargement of the model's size. \citet{brown2020language} also found that enhancing the model size can boost its zero-shot capabilities. These findings suggest that scaling the model can enhance both the weights and context components, and the model employs these two components to address the problem.

3.\cite{wei2023larger} carried out research on a two-class classification issue. They conducted experiments in which they altered a certain percentage of labels in the context examples to ascertain if the model's prediction would also change. If a change was observed, it would imply that the prediction relies on the context components. If no change was noticed, the prediction would be considered to depend on the weights component. Their results were intermediate, suggesting that both weights and context components contribute. Additionally, they found that enhancing the model size increases the impact of in-context examples.

\subsection{Related works for understanding Transformer}
\textbf{Analysis of Transformer} \quad The analysis of Transformers can be broken down into two main components: examining the expressibility of Transformers and comprehending the mechanisms of learned Transformers. To analyze the expressibility of Transformers, a common approach is to determine if they can solve specific problems by constructing appropriate weights. \citet{giannou2023looped} demonstrates that Transformers can function as Turing machines, while \citet{liu2022transformers} shows that they can learn shortcuts to solve automata problems. In addition to expressibility, researchers have also investigated the mechanisms behind learned Transformers. \citet{bietti2023birth} examines Transformers from a memory standpoint, and \citet{tian2023scan} focuses on single-layer Transformers. While the analysis of Transformers is crucial to our work, our ultimate goal differs; we aim to bridge the gap between representation learning and in-context learning.

\textbf{Exploration of representation within Transformer.} \quad Owing to the widespread use of Transformers, numerous studies \citep{li2022emergent,voita2020information} seek to investigate their internal representations as a means of comprehending their functionality. The most prevalent approach involves utilizing probe models and tasks to discern the information stored within these representations \citep{voita2020information, schouten2022probing}. Taking a different perspective, \citet{voita2019bottom} explores the flow of information across Transformer layers and how this process is influenced by the selection of learning objectives. 
Our work shares similarities with these studies in that we employ the probe method to examine representations. However, our focus differs in that we do not concentrate on the semantic meaning within the representation. Instead, we investigate how the weights and in-context information impact representation.

% \textbf{Understanding in-context learning} \quad

\section{Proof of Proposition \ref{prop:decomposite} }
\begin{proposition}
Given $y_p$, probability of $\sP(y_p|\rvx_p,s_c)$ can be decomposite as:
\begin{equation}
\sP(y_p|\rvx_p,s_c)=\sum_{v_p,m,e_h}\sP(y_p|v_p,m,e_h)\sP(v_p|\rvx_p)\sP(e_h|s_c,m)\sP(m|s_c).
\end{equation}
\end{proposition}

\begin{proof}
    \begin{equation}
    \begin{aligned}
        \sP(y_p|\rvx_p,s_c) &=\sum_{v_p,m,e_h}\sP(y_p,v_p,m,e_h|\rvx_p,s_c)
        \\&=\sum_{v_p,m,e_h}\sP(y_p|\rvx_p,s_c,v_p,m,e_h)\sP(v_p,m,e_h|\rvx_p,s_c,)
        \\&=\sum_{v_p,m,e_h}\sP(y_p|v_p,m,e_h)\sP(v_p|\rvx_p,s_c,m,e_h)\sP(m,e_h|\rvx_p,s_c)
        \\&=\sum_{v_p,m,e_h}\sP(y_p|v_p,m,e_h)\sP(v_p|\rvx_p)\sP(e_h|s_c,m)\sP(m|s_c),
    \end{aligned}
    \end{equation}
    where the first equation is due to the law of total probability, the third equation is leverages the formular $\sP(y_p|\rvx_p,s_c,v_p,m,e_h)=\sP(y_p|v_p,m,e_h)$.

\end{proof}

\section{Contruction analysis of why the weights component is related to the learning plateaus}
\label{sec:contruction_analysis}

\begin{figure}[h] 
  \centering  
  
  \includegraphics[width=0.6\textwidth]{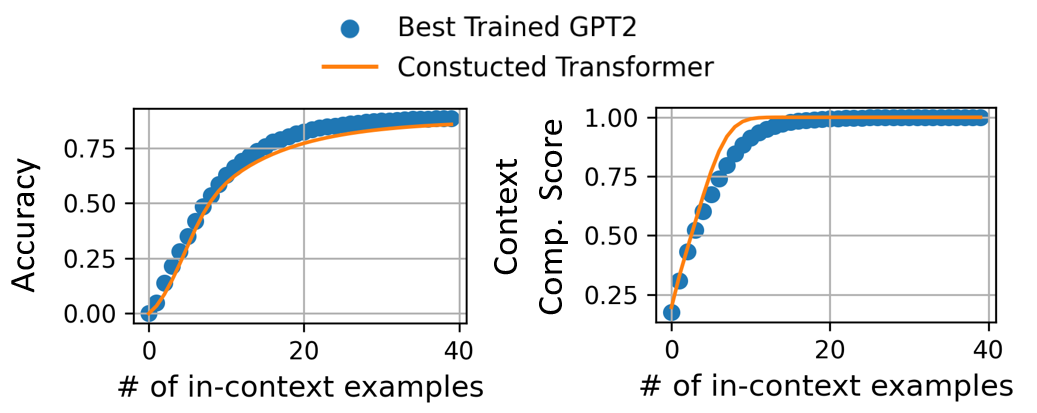}   
  \caption{The constructed Transformer can match the performancce of trained Transformer ($D_{\fx \land \rnd}$ setting) in experiment part}  
\label{fig:verify_theory}
\end{figure}

\textbf{Intuition} \quad In this section, our objective is to comprehend the connection between the weights component of the model and the occurrence of learning plateaus. Owing to the challenges involved in dissecting the training dynamics of the Transformer, we turn to constructive analysis as a methodology. By examining a scenario where a Transformer possesses an effective weights component at a specific layer, and only a small number of additional layers are required to attain proficient in-context learning performance, we can indirectly infer insights about their relationship.

\textbf{Notation} \quad The position embedding is denoted as $\rvp_i=(0,\cdots0,1,0,\cdots) $, where we only have value 1 at $i$-th position and 0 others. The weights for the attention operation of $l$-th layer and $c$-th head in Transformer is denoted as $\rmW^{(l,c)}_Q$, $\rmW^{(l,c)}_K$ and $\rmW^{(l,c)}_V$. The weights of the forward layer in the Transformer are denoted as $\rmW^{(l)}_1,\rmW^{(l)}_2$. We use $E$ to denote all possible values of the factor $e$. we denote $\rvy_i$ as the one hot version of $y$. The vector with all zero values is denoted as $\bzero\triangleq (0, \cdots,0)$.
We consider the naive Transformer~\citep{vaswani2017attention}. The hidden representation of token $i$ in Transformer is denoted as $\rvh_i  \in \mathbb{R}^d$. The hidden representation of $l$-th layer is denoted as $\rmH^{(l)} = [\rvh_1^{(l)},\cdots, \rvh_{2L}^{(l)}]\trans \in \mathbb{R}^{2L \times d}$. Given a input token $\rvx$, we denote $\rvh_{\rvx}^{(l)}$ as its corresponding representation at layer $l$. The factor value of this token is denoted as $v_{\rvx}$. The factor value of the corresponding factor $e$ is denoted as $v_{\rvx}^{(e)}$. The set of all possible values of factor $e$ is denoted as $V_e$. The size of the set is denoted as $|V_e|$.

In our analysis, we explore a modified, more relaxed variant of the Transformer model. The rationale behind this relaxation is underpinned by evidence suggesting that 1) employing the ReLU activation function in the feed-forward layers can achieve results comparable to the original model~\citep{press2019improving}, and 2) the softmax operation may not be essential for the functioning of the Transformer~\citep{wiegreffe2019attention, brunner2019identifiability, richter2020normalized}. The relaxed Transformer is defined as follows:

\begin{definition}
    (Transformer) One layer of the Transformer contains one attention layer and one MLP layer. The calculation of the Attention Layer is
    \begin{equation}
    \operatorname{Attn}^{(l)}(\rmH^{(l)})=\rmH^{(l)}+\sum_{c=1}^{C} \sigma \left(\rmH^{(l)} \rmW_Q^{(l,c)}(\rmH^{(l)} \rmW_K^{(l,c)})^{\mathrm{T}} \right)\rmH^{(l)} \rmW_V^{(l,c)} \rmW_O^{(l,c)}.    
    \end{equation}
    And the calculation of MLP layer is
    \begin{equation}
    \rmH^{(l+1)}=\operatorname{Attn}^{(l)}(\rmH^{(l)})+\operatorname{Relu}(\operatorname{Attn}^{(l)}(\rmH^{(l)})\rmW_1^{(l)})\rmW_2^{(l)}.  
    \end{equation}
Here we consider relaxed case where $\sigma=\operatorname{Id}$.
\end{definition}

Remember that a good weights component implies that we have the ability to deduce $\sP(v_{\rvx}|\rvx)$ based on this component. To streamline the construction, we introduce a stronger assumption: the concept of a perfect weights component. Unlike the definition of a good weights component, which necessitates that the Transformer encapsulates information about $\sP(v_{\rvx}|\rvx)$, a perfect weights component also demands that this information should be readily accessible. If the representations of the images corresponding to different factor values are situated in distinct orthogonal bases, then the factor values' information can be easily decoded. Drawing on this insight, we propose the following definition.

\begin{definition}
    \label{as:token_rep}
    (Perfect weights component) If a Transformer has perfect weights component in layer $l$,  then for all factor $e$, any $i,j$, exists $\rmW_e \in \mathbb{R}^{d \times |V_e|}$ such that $\rvf_{\rvx_1}^{(e)}\cdot \rvf^{(e)}_{\rvx_2}=1$ only when $v^{(e)}_{x_1}=v^{(e)}_{x_2}$, else we have $\rvf_{\rvx_1}^{(e)}\cdot \rvf^{(e)}_{\rvx_2}=0$, where $\rvf_{\rvx}^{(e)}=\rmW_e \rvh_{\rvx}^{(l)}$.
\end{definition}

Under the assumption of a perfect weights component, we can enhance the Transformer by adding at most three extra layers that are specifically designed to learn the context component. The detailed results of this construction are as follows:

\begin{proposition}
    \label{prop:transformer_perfect_token_rep} We consider the data with $n_e$ factors and each factor has $n_v$ values in $D_{\rnd}$ setting.   For causal Transformer with the number of heads larger or equal the number of factors with the hidden size $\mathcal{O}(n_e n_v+L)$, if the Transformer can learn a perfect weights component in layer $k$, then it can learn a representation given $i$ in-context examples with context comp. score $\srs_i= (1-\srs_{i-1})s_i+\srs_{i-1}$ and $\srs_0=s_0$ at layer $k+2$, where $ s_i= 1-\sum_{j=0}^{i} \tbinom{i}{j}  \sum_{k=2}^{|E|} \tbinom{|E|}{k} \frac{k-1}{k} \left( \frac{(n_v-1)^{i-j}}{n_v^{i}} \right)^k \left( 1-\frac{(n_v-1)^{i-j}}{n_v^{i}} \right)^{|E|-k}$,  and we can obtain the accuracy as $\cls_i = \frac{(n_v-1)(n_v^{i-1}-(n_v-1)^{i-1})}{n_v^i}\srs_i+\frac{1}{n_v}$ at $k+3$ layers. 
\end{proposition}

\textbf{The constucted Transformer achieve significant performance.} \quad To illustrate that the construction is meaningful, we compare the performance of the constructed model with the previously trained model. We choose the performance of the model trained on $D_{\fx \land \rnd}$ settings as it is the best-performed model on $D_{\rnd}$ setting. We find that the constructed model achieved a comparable performance as that of the trained model. These results indicate that our construction is meaningful.
\subsection{Proof Sketch}

\textbf{a) Contruction of the Transformer.}

We divided the contruction into two steps. The first step is to estimate the factor in the sequence and the second step is to estimate the label based on the discovered hidden factor. Given the sequence $s=\{(\rvx_1,y_1),\ldots,(\rvx_L,y_L)\}$, we short $\rvf_{\rvx_j}$ as $\rvf_j$.

\textit{1. Estimate the hidden factor.} \quad 
According to the perfect weights component assumption, for any $j$, we can project the $j$-th token feature into the space $\rvf^{(e)}_j$. Assuming the $j$-th token is not the prompt token, then we have its label information $y_j$. Then, for $i$-th token, where $i<j$, obviously $i$ is also not the prompt token. As a result, we also have the information about $y_i$. If a factor $e$ is the hidden factor, then we would expect $\rvy_i\cdot\rvy_j=\rvm(v^{(e_h)}_i)\cdot\rvm(v^{(e_h)}_j)=\rvv^{(e_h)}_i\cdot\rvv^{(e_h)}_j=\rvf_i^{(e_h)}\cdot \rvf_j^{(e_h)}$, where $\rvv,\rvy,\rvm$ is the corresponding one-hot version of $v,y,m$. Therefore, if we can find $e$ such that $\rvf_i^{(e)}\cdot\rvf_j^{(e)}$ can have a same value as $\rvy_i\cdot\rvy_j$, for all $i$, then $e$ can be predicted as hidden factor. Based on this intuition, in the construction, we focus on finding a way to compare the value between $\rvf_i^{(e)}\cdot\rvf_j^{(e)}$ and $\rvy_i\cdot\rvy_j$. All these operations are done in the first layer.

\textit{2. Block unrelated information.} \quad After finding the hidden factor, the next step is to block the information that is unrelated to the hidden factor. Blocking unrelated information can remove the influence of it and simplify the following steps.  We add a large negative value to the positions of the representation that is unrelated to the hidden factor. Then, through the Relu operation, all these negative values will be removed.  These operations are placed in Layer 2. 

\textit{3. Predict $y$.} \quad The final step is to predict the $y_p$ for the prompt sample $\rvx_p$. The challenge here is that we don't know the mapping function $m$ that bridges the factor value and the corresponding label. Consider the relation that for any $i,j$ satisfying $i<j$, we have $ \rvv_i^{(e_h)}\cdot \rvv_j^{(e_h)}=\rvf_i^{(e_h)}\cdot \rvf_j^{(e_h)}=\rvy_i\cdot \rvy_j$. Therefore, if we can find a $i$, satisfying that $\rvf_i^{(e_h)}\cdot \rvf_j^{(e_h)}=1$, then we have $y_i=y_j$. Then, we can copy $y_i$ to the representation of $j$-th token and use the final linear layer to output $y_i$ as the prediction for $y_j$.

\textbf{b) Analyzing the performance.}

Given the perfect weights component assumption, there are two source that cause the error prediction. We give the separate analysis below.

\textit{1. Fail to find the correct hidden factor.} Finding the correct hidden factor is essential to make the correct prediction. The correct hidden factor cannot be inferred if there are more than two factors such that the values of $\rvf_i^{(e)}\cdot\rvf_j^{(e)}$ and $\rvy_i\cdot\rvy_j$ can be matched according to the construction of first layer. 

\textit{2. Make error prediction and give the correct factor.} Given the hidden factor, it is also possible that we cannot make the correct prediction if we fail to associate the correct label  the factor value. This will happen if all the context examples don't contain the same factor value of the hidden factor as the prompt image.

\subsection{Proof of useful lemma}

In this section, we proof some useful lemma for the proof. The lemma \ref{lm:copy_past} indicates that the attention head can copy part of the representation from its previous token into current token. The lemma \ref{lm:sub_mlp} indicates that there exists a construction to make the MLP to only operate on the part of its input.

\begin{lemma}
    \label{lm:copy_past}
   One attention head can implement copy and past behavior.
\end{lemma}
\begin{proof}
    According to the definition of $\rvp_i$, we have $\rvp_i \cdot \rvp_j=0$ if $i\neq j$, otherwise, we have $\rvp_i \cdot \rvp_j=1$.
    We denote
    
    \[\rmM =  
        \begin{bmatrix}  
        0 & 0 &  \cdots &0  & 1 \\  
        1 & 0 &  \cdots &0 & 0 \\  
        0 & 1 &  \cdots & 0 &  0 \\
        \vdots  & \vdots & \ddots& \vdots & \vdots \\  
          
        0 & 0 &  \cdots & 1 & 0  
        \end{bmatrix}.  
    \]
    Then we have $ \rvp_i \rmM= \rvp_{i-1}$.
    For $j>i$, we denote the value of $2j$-th token as $\rvh_{2j}=(\bzero,\bzero,\bzero,\bzero,\bzero,\bzero,\bzero,\rvp_j)$ and $2i$-th token as $\rvh_{2i}=(\rvh'_{i},\bzero,\bzero,\bzero,\bzero,\bzero,\bzero,\rvp_i)$. 
    If we want to copy the value of $2i$-th token to the value of $2j$-th token, we can set the query matrix as $\rmW_Q=(\bzero,\bzero,\bzero,\bzero,\bzero,\bzero,\bzero,\rmM^{j-i})$, the key matrix as $\rmW_K=(\bzero,\bzero,\bzero,\bzero,\bzero,\bzero,\bzero,\rmI)$ and value matrix as $\rmW_V=(\rmW'_V,\bzero,\bzero,\bzero,\bzero,\bzero,\bzero,\bzero)$. Then we have 
    \begin{equation}
         \rvh_{2i}\trans \rmW_Q \cdot  \rvh_{2a}\trans \rmW_K= \rvp_{i} \cdot \rvp_{a} =
        \left\{ 
        \begin{array}{cc}
            1  & a \neq j \\
            0   & a = j
        \end{array} 
        \right.
    \end{equation}
    Therefore, the $2j$-th token can only attend to the token with the position embedding $\rvp_i$.
    If $\rvh_{2i-1}=(\bzero,\bzero,\bzero,\bzero,\bzero,\bzero,\bzero,\rvp_i)$, we have the value of $\rvh_j$ after attention as $\rvh_j^{attn}=((\rvh'_i)\trans  \rmW_V,\bzero,\bzero,\bzero,\bzero,\bzero,\bzero ,\rvp_j)$.
    By setting $\rmW_V$ as different value, we can copy different part information of $i$-th to $j$-th token. Then the lemma is held.
\end{proof}

\begin{lemma}
    \label{lm:sub_mlp}
    For the input $\rvh=(\rvh_1,\rvh_2 ,\rvh_3 )$, where $\rvh_i \in \mathbb{R}^{d_i}$ and $d_1+d_2+d_3=d$, for all $MLP_s(\rvh)=\rmW_2' \relu(\rmW_1' \rvh_2):\mathbb{R}^{d_2} \to \mathbb{R}^{d_2}$, there exists $MLP(h)=\rmW_2 \relu(\rmW_1 \rvh): \mathbb{R}^d \to \mathbb{R}^d$, such that $MLP(\rvh)=(\rvh_1, MLP_s(\rvh_2) , \rvh_3)$.
\end{lemma}
\begin{proof}
    Obviously, for any $\rmW_1'$, there exists $\rmW_1$, such that $\rvh^{(a)}\triangleq \rvh\rmW_1=(\rvh_1,-\rvh_1,(\rmW_1' \rvh_2) ,\rvh_3 \trans,-\rvh_3 )$.
    
    Obviously, for any $\rmW_2'$, There exists $\rmW_2$, such that $\rvh^{(b)}=\rmW_2 \relu(\rvh^{(a)})=((\relu(\rvh_1)+\relu(-\rvh_1)) , (\rmW_2' \relu(\rmW_1')) , (\relu(\rvh_3)+\relu(-\rvh_3)))=(\rvh_1, MLP_s(\rvh_2), \rvh_3)$ 
\end{proof}

\subsection{Construction of Transformer}
 Without loss of generality, we assume the representation of the Transformer in layer $k$ is in a form that
$\rvh_{2i-1}^{(k)}=(\rvf_{i}, \bzero,\bzero,\bzero,\bzero,\bzero,\bzero,\rvp_{i}  )\trans$ and 
$\rvh_{2i}^{(k)}=(\bzero, \rvy_i,\bzero,\bzero,\bzero,\bzero,\bzero,\rvp_{i} )\trans$ (Remind that one sample will take two token, one for $x$ and one for $y$). Because the representation usually lies in low-dimension space, a simple linear layer can transfer the representation in our defined sparse form. Moreover, it is natural to assume that the position information is stored in the representation since it is given in the input and is essential for attention.

The consider the operations of Transformer in different layers.

     \textbf{1) ** Layer 1 **}

    Because we assume that $\rvh_{2i-1}^{(k)}$ is a perfect token representation, then there exists $\rmW_e$, such that $\rvh_{2i-1}^{(k)} \rmW_e=\rvf^{(e)}_k$, where $\rvf_i^{(e)}$ satisfies that $\forall e,i$, we have $\rvf_j^{(e)} \cdot \rvf_i^{(e)}=1$ only when $v_i^{(e)}=v_j^{(e)}$ else $\rvf_j^{(e)} \cdot \rvf_i^{(e)}=0$.

    \textbf{Step 1, we use each attention head to obtain the matching information of each factor.} 
    
    We first consider the query token at the position $2i-1$
    And we assign $\rmW_Q^{(l,k)}=\rmW_K^{(l,k)}=\rmW_{e_k}$ and  $\rmW_V^{(l,k)}=(\bzero,\bzero,\bzero, \bzero, \bzero, \bzero, \bzero, \rmI)\trans$ so that $(\rvh_i^{(l)})\trans \rmW_V^{(l,k)} = \rvp_i$. Then, we have

    \begin{equation}
        \rvb_{e_k}=\sum_{a=1}^{2i-2}( \rvh_i\trans \rmW_Q^{(l,k)} \cdot  \rvh_a\trans \rmW_K^{(l,k)} )  \rvh_i\trans \rmW_V^{(l,k)} = \sum_{a=1}^{i-1} \rvp_a (\rvv_a^{(e_k)}\cdot \rvv_i^{(e_k)}),
    \end{equation}
    where $\rvv_i^{(e)}$ is the one-hot vector of $v_i^{(e)}$.
    We denote $\text{base}=(2^0,2^1,\cdots,2^L)\trans$ and $\rvu_{2i-1}=(\{\text{base} \cdot \rvb_{e_1}, \cdots, \text{base} \cdot \rvb_{e_{n_e}}\}$.
    Obvious, there is $\rmW_O^{(l,k)}$ such that  $\sum_{k=1}^{n_e} \rvb_{e_k} \rmW_O^{(l,k)}=(\bzero,\bzero,\rvu_{2i-1},\bzero,\bzero,\bzero,\bzero,\bzero)$. 

    Then, we consider the token at position $2i$ as query token. We assign $\rmW_Q^{(l,n_e+1)}=\rmW_K^{(l,n_e+1)}=(\bzero,\rmI,\bzero,\bzero,\bzero,\bzero,\bzero,\bzero)\trans$ and $\rmW_V^{(l,n_e+1)}=(\bzero, \bzero, \bzero, \bzero,\bzero,\bzero, \bzero, \rmI)\trans$. We have

    \begin{equation}
        \rvb_{y}=\sum_{a=1}^{2i-1}( \rvh_i\trans \rmW_Q^{(l,n_e+1)} \cdot  \rvh_a\trans \rmW_K^{(l,n_e+1)} )  \rvh_i\trans \rmW_V^{(l,n_e+1)} = \sum_{a=1}^{i-1} \rvp_a(\rvy_a\cdot \rvy_i)
    \end{equation}
    Obvious, there is $\rmW_O$ such that  $ \rvb_{y} \rmW_O^{(l,n_e+1)}=(\bzero,\bzero,\bzero,\rvu_{2i},\bzero,\bzero,\bzero,\bzero,\bzero)$, where $\rvu_{2i}=\lbrace \operatorname{base}\cdot \rvb_y, \cdots, \operatorname{base}\cdot \rvb_y \rbrace$. 

    The $n_e+1$ head doesn't affect the token $2i-1$, because value of the  $\rvh_{2i-1}$ is $\bzero$. As a result, we have
    $\rvh_{2i-1}=\rvh_{2i-1}+\sum_{k=1}^{n_e} \rvb_{e_k} \rmW_O^{(l,k)}=(\rvf_i,\bzero,\rvu_{2i-1},\bzero,\bzero,\bzero,\bzero,\rvp_i)$ after the operation. Similarily, because the first $n_e$ head dosen't affect the value of $\rvh_{2i}$, we have $\rvh_{2i}=(\bzero,\rvy_i,\bzero,\rvu_{2i},\bzero,\bzero,\bzero,\rvp_i)$ after the operation. 
    
    Note that $\text{base} \cdot \rvb_{e_k}$ has the property that $\text{base} \cdot \rvb_{e_k}=\text{base} \cdot \rvb_{e_{k'}}$ if and only if $\rvb_{e_k}=\rvb_{e_k'}$. This result indicates that all the context examples that have same factor value of factor $e_k$ with the sample $i$ also has the same factor value of factor $e_{k'}$ as sample $i$.  \textbf{Therefore, we denote $\rvu$ as the matching information.} If $\text{base} \cdot \rvb_{e_k}=\text{base} \cdot \rvb_{y}$, we can infer that the factor value of $e_k$ has a similar pattern with the label. Therefore, $e_k$ is regard as the possible hidden factor.

    Step 2:  compare the $\rvu_{2i-1}$ and $\rvu_{2i}$ to infer possible hidden factor.

    For embedding of $\rvh_{2i}$, using the copy past of Lemma \ref{lm:copy_past}, we can obtain $\rvh_{2i}=(\bzero,\rvy_i,\bzero,\rvu_{2i},\rvu_{2i-1},\bzero,\bzero,\rvp_i)$. (By setting the copy position as $\rvp_i$ and therefore the operation will only influence $y$ token.)
    According to Lemma \ref{lm:sub_mlp}, there exists $\rmW_1^{(l)},\rmW_2^{(l)}$, such that we have $\rvh_{2i}=(\bzero,\rvy_i,\bzero,\rvu_{2i},\rvu_{2i-1},\rvm_{2i},\bzero,\rvp_i)$, where
    $\rvm_{2i}=\relu(\rvu_{2i}-\rvu_{2i-1})+\relu(\rvu_{2i-1}-\rvu_{2i})$. Recall that $\rvh_{2i-1}=(\rvf_i,\bzero,\rvu_{2i-1},\bzero,\bzero,\bzero,\bzero,\rvp_i)$. because all the corresponding terms of $\rvh_{2i-1}$ are $\bzero$, this operation won't impact the value of it. This copy past operation can be put into a same layer as the pervious operations is because in this operation we mainly copy the infomation from $2i-1$ token to $2i$-token. Because $2i-1$ precede $2i$, the operations of $2i-1$ is finished before the copy past operation happens.
    
    % $-\text{base}\cdot \rvp_i \relu(\rvu_{2i-1}-\text{base}\cdot \rvp_i+1)$.

    The value of $\rvm_{2i}$ has the property that the $k$-th position in $\rvm_{2i}$ is equal to $0$ if the values of $k$-th position of  $\rvu_{2i-1}$ and $\rvu_{2i}$ are equal. After this operation, we have $\rvh_{2i}=(\bzero,\rvy_i,\bzero,\rvu_{2i},\rvu_{2i-1},\rvm_{2i},\bzero,\rvp_i)$ and $\rvh_{2i-1}=(\rvf_i,\bzero,\rvu_{2i-1},\bzero,\bzero,\bzero,\bzero,\rvp_i)$.

     \textbf{2) ** Layer 2 **}

    \textbf{Blocking the information according to $\rvm$.}

    \textbf{First attention head:} for $\rvy$ token, at position $2i$, we apply Lemma \ref{lm:copy_past} to copy $\rvm_{2i-2}$ from $\rvh_{2i-2}$ to $\rvh_{2i}$.  Due to the weights sharing of attention, this yield a iterative effect. We denote $\rvm_{2i-1}'=  2\rvm_{2i-3}'+\rvm_{2i-2}$ and $\rvm_{2i}'=\rvm_{2i-1}+\rvm_{2i}$.
    Therefore, we have $\rvh_{2i}=(\bzero,\rvy_i,\bzero,\rvu_{2i},\rvu_{2i-1},\rvm_{2i}',\bzero,\rvp_i)$. Because of weights sharing, we have $\rvh_{2i-1}=(\rvf_i,\bzero,\rvu_{2i-1},\bzero,\bzero,\rvm_{2i-1}',\bzero,\rvp_i)$ .

    \textbf{Second attention head:} In this layer, for $\rvy$ token, at position $2i$, we apply Lemma \ref{lm:copy_past} to copy $\rvf_i$ from $\rvh_{2i-1}$ (This operation only affects $\rvy$ tokens). We have $\rvh_{2i}=(\rvf_i,\rvy_i,\bzero,\rvu_{2i},\rvu_{2i-1},\rvm_{2i}',\bzero,\rvp_i)$.
    
    \textbf{MLP Layer:} We denote $\rvf_{i,x}' \triangleq (\rvf_{i}^{(e_j)}-M \rvm_{2i-1}'[1],\cdots,\rvf_i^{(e_j)}-M \rvm_{2i-1}'[n_e])$ and $\rvf_{i,y}' \triangleq (\rvf_{i}^{(e_j)}-M \rvm_{2i}'[1],\cdots,\rvf_i^{(e_j)}-M \rvm_{2i}'[n_e])$. $M$ is a large constant value. In $\rvf_{i,x}'$, we will block the information of $j$-th factor if $\rvm_{2i-1}'[j]>0$. $\rvm_{2i-1}'[j]<0$ if and only if $\forall ~ k<i,~ \rvm_{2k}[j]=0 $. The same for $\rvf_{i,y}'$ 
    In MLP, we calculate $ \relu( \rvh_{2i-1} \trans \rmW_1^{(l+2)})\rmW_1^{(l+2)}=(\rvf_i,\rvf_i')\rmW_1^{(l+2)}=(\rvf_i'-\rvf_i,\bzero,\bzero,\bzero,\bzero,\bzero,\bzero, \bzero)$. Then, we have $\rvh_{2i-1}= \relu( \rvh_{2i-1} \trans \rmW_1^{(l+2)})\rmW_1^{(l+2)}+\rvh_{2i-1}=(\rvf_{i,x}',\bzero,\rvu_{2i-1},\bzero,\rvm_{2i-1}',\bzero,\bzero,\rvp_i)$. And similar, we have $\rvh_{2i}=(\rvf_{i,y}',\rvy_i,\bzero,\rvu_{2i},\rvu_{2i-1},\rvm_{2i}',\bzero,\rvp_i)$.

    \textbf{3) ** Layer 3 **}

     This layer obtains the logit of the new sample by comparing the similarity between the unblocked feature of this sample and the in-context sample.

    Setting $\rmW_Q^{(l+3,1)}=\rmW_K^{(l+3,1)}=(\rmI,\bzero,\bzero,\bzero,\bzero,\bzero,\bzero,\bzero)$, we have $\rvh_i\trans \rmW_Q^{(l+3,1)} = \rvh_i \trans \rmW_K^{(l+3,1)}= \rvf_i'$  . Setting $\rmW_V^{(l+3,1)}=(\bzero,\rmI,\bzero,\bzero,\bzero,\bzero,\bzero,\bzero)$ such that $\rvh_{2i} \trans \rmW_V^{(l+3,1)}=\rvy_i$ and $\rvh_{2i-1} \trans \rmW_V^{(l+3,1)}=0$.

    For position $2i-1$, we have
    \begin{equation}
        \text{Logit}=\sum_{a=1}^{2i-2}( \rvh_i\trans \rmW_Q^{(l+3,1)} \cdot  \rvh_a\trans \rmW_K^{(l+3,1)} )  \rvh_i\trans \rmW_V^{(l+3,1)} = \sum_{a=1}^{i-1} (\rvf_{i,x}' \cdot \rvf_{a,y}') \rvy_a'.
    \end{equation}
    
    Note that value $\rvf_{i,x}' \cdot \rvf_{a,y}'$ is equal to the number of unblocked factors (both unblocked) that have the same value between $a$-th sample and $i$-th sample
    Obviously, there is a $W_O^{(l+3,1)}$ such that 
    $\rvh_{2i-1}=(\rvf_{i,x}',\bzero,\rvu_{2i-1},\bzero,\bzero,\bzero,\rvm_{2i-1}',\text{Logit},\rvp_i)$.

    Finally, we output Logit using the prediction head.
    
    \subsection{Performance analysis}

    Here, we will analyze the sequence representation score and the in-context learning accuracy of our constructed model.

    \textbf{1) **context comp. score**}

    \paragraph{Propability for same factor value between  in-context examples and prompt sample} The probability for an in-context example having the same value of a factor as the prompt sample is $\frac{1}{n_v}$ and the probability of having different values is $\frac{n_v-1}{n_v}$. Therefore, given $i$ samples, the probability for $j$ samples have the same value of a factor as the prompt sample is $ \tbinom{i}{j} \frac{(n_v-1)^{i-j}}{n_v^{i}}$. 
    
    \paragraph{Probability for connot distinguish factors} Given $i$ in-context examples, we cannot distinguish $k$ factors to decide which one is the hidden factor if the $k$ factors satisfying that of $\forall~ e_1,e_2\in E_k,~ (x,y) \in s_c$, we have $v^{(e_1)}_x=v_p^{(e_2)}\Leftrightarrow v^{(e_2)}_x=v_p^{(e_2)}$, where $E_k$ is the set of these $k$ factors, $s_c$ is in-context examples, and $v$ is factor value.

    Given $i$ in-context examples, the probability for that we cannot distinguish $k$ factors is
    \begin{equation}
        \tbinom{|E|}{k} \sum_{j=0}^i \tbinom{i}{j}  \left( \frac{(n_v-1)^{i-j}}{n_v^{i}} \right)^k \left( 1-\frac{(n_v-1)^{i-j}}{n_v^{i}} \right)^{|E|-k}.
    \end{equation}
    
    \paragraph{context comp. score} When we cannot distinguish the hidden factor from $k$ factors, the probability of predicting wrong results is $\frac{k-1}{k}$. Combining the results above, we obtain the error:
    \begin{equation}
        \sum_{k=2}^{|E|} \tbinom{|E|}{k} \sum_{j=0}^{i} \tbinom{i}{j}   \frac{k-1}{k} \left( \frac{(n_v-1)^{i-j}}{n_v^{i}} \right)^k \left( 1-\frac{(n_v-1)^{i-j}}{n_v^{i}} \right)^{|E|-k}.
    \end{equation}
    The probability of giving the right prediction is 
    \begin{equation}
        s_i= 1-\sum_{j=0}^{i} \tbinom{i}{j}  \sum_{k=2}^{|E|} \tbinom{|E|}{k} \frac{k-1}{k} \left( \frac{(v-1)^{i-j}}{v^{i}} \right)^k \left( 1-\frac{(v-1)^{i-j}}{v^{i}} \right)^{|E|-k}.
    \end{equation}
    In the constructed Transformer, we will autoregressively combine the results of the previous prediction (Corresponding to Layer 2). We have:

    $$\srs_i= (1-\srs_{i-1})s_i+\srs_{i-1},$$
    where $\srs_0=s_0$.

    \textbf{2) **Accuracy**}

    The copy-past mechanism is used to predict the answer to the prompt example (Corresponding to layer 3). For the copy-past mechanism, having an in-context example with the same prediction result as the prompt example is necessary. When we correctly predict the hidden factor, the probability to predict correctly is $1-(\frac{n_v-1}{n_v})^i$. When we predict a wrong hidden factor, the probability is $\frac{1}{n_v}$. Combine the two above, we obtain the accuracy as follows
    \begin{equation}
        (1-(\frac{n_v-1}{n_v})^i)\srs_i+ \frac{1}{n_v}(1-\srs_i)=\cls_i = \frac{(n_v-1)(n_v^{i-1}-(n_v-1)^{i-1})}{n_v^i}\srs_i+\frac{1}{n_v}
    \end{equation}

\end{document}